\documentclass{article}
\usepackage{iclr2026_conference,times}

\usepackage{amsmath,amsfonts,bm}

\def\eqref#1{equation~\ref{#1}}
\def\1{\bm{1}}

\DeclareMathAlphabet{\mathsfit}{\encodingdefault}{\sfdefault}{m}{sl}
\SetMathAlphabet{\mathsfit}{bold}{\encodingdefault}{\sfdefault}{bx}{n}

\newcommand{\R}{\mathbb{R}}

\usepackage{hyperref}
\usepackage{url}
\usepackage{graphicx}
\usepackage{booktabs}
\usepackage{amsmath,amssymb,amsthm}
\usepackage{algorithm,algpseudocode}

\graphicspath{{figures/}}

\newtheorem{theorem}{Theorem}
\newtheorem{proposition}{Proposition}
\newtheorem{lemma}{Lemma}
\newtheorem{corollary}{Corollary}
\theoremstyle{definition}
\newtheorem{definition}{Definition}
\renewcommand{\R}{\mathbb{R}}
\newcommand{\dq}{\mathrm{dq}}
\newcommand{\VC}{\mathrm{VC}}

\title{Spectral Anatomy of Quantum Gaussian Process Kernels}

\author{%
\normalfont\normalsize
Jian Xu$^{1,2}$\thanks{\texttt{jian.xu@riken.jp}},
Chao Li$^{2}$, Guang Lin$^{2}$, Yuning Qiu$^{2}$, Delu Zeng$^{3}$, John Paisley$^{4}$, Qibin Zhao$^{1}$\\[0.3em]
$^{1}$RIKEN iTHEMS\quad $^{2}$RIKEN AIP\quad $^{3}$South China University of Technology\quad $^{4}$Columbia University
}

\newcommand{\Seff}{S}
\newcommand{\reff}{r_{\mathrm{eff}}}
\newcommand{\Kq}{K_{q}}

\iclrfinalcopy  % preprint version: show authors and "Published as..." header.

\begin{document}

\maketitle

% preprint: clear ICLR "Published as..." running headers AFTER \maketitle
\lhead{}
\chead{}
\rhead{}
\renewcommand{\headrulewidth}{0pt}
\thispagestyle{fancy}

\begin{abstract}
Two recent results have reshaped quantum Gaussian processes (QGPs). On the
one hand, \citet{lowe2025assessing} rule out the exponential speedups
claimed by HHL-based QGP regression in the typical, well-conditioned
regime; on the other, an independent line of work shows that highly
expressive quantum kernels suffer posterior pathologies that break
Bayesian optimization. We show that these seemingly unrelated phenomena
are governed by the same quantity: the normalized spectral entropy
$\Seff(K)/\log n$ of the kernel Gram matrix. We prove a Cauchy--Schwarz
tail bound on Nystr\"om approximation error, a finite-sample
variance-contraction identity in terms of Bach's degrees of freedom
$d_\sigma(K)$, and a characterization of the \emph{target-dependent}
optimal entropy via the intrinsic dimension of the target in the kernel
eigenbasis. Empirically, the diagnostic is kernel-agnostic: hardware-efficient,
matchgate, IQP \emph{and} RBF/Mat\'ern/RFF/deep-kernel families all
collapse onto identical $\Seff/\log n$ curves on dequantization, ECE, and
variance-contraction panels. The NLL sweet spot lives at high entropy
for smooth targets and at low entropy for band-limited quantum-data
targets. The diagnostic transfers from simulator to IBM Heron hardware with
median absolute error $3.2\%$ and mean $5.2\%$ in $\Seff/\log n$ across
$24$ configurations at $n_q = 4$, with matchgate and IQP within
$5\%$ mean and a single HE configuration returning a $30\%$ outlier
that drops to $0.5\%$ on rerun (attributed to calibration drift); the
same diagnostic transfers to a second Heron backend (mean
error $2.7\%$) and to a $n_q = 6$ scale-up on the original backend
(mean error $1.7\%$). No error mitigation is applied throughout.
\end{abstract}

\section{Introduction}
\label{sec:intro}

Quantum Gaussian processes (QGPs) promise two things that classical GPs cannot
easily deliver: kernels that encode physics-motivated symmetries
\citep{jager2026provable,rapp2024quantum}, and---in principle---asymptotic
runtime speedups for Bayesian inference via quantum linear solvers
\citep{zhao2019bayesian,zhao2019quantum,chen2022quantum}. The past eighteen
months have however brought a sharp reframing.

First, \citet{lowe2025assessing} proved that the kernel matrices
arising in many QGP regression settings---specifically, those for which
HHL-based exponential speedups had previously been claimed---have
condition number scaling $\kappa(K) = \Omega(n)$, so that the HHL pipeline
delivers at most polynomial advantage in those regimes. Second, an independent
line of work on quantum-kernel concentration
\citep{thanasilp2024exponential,kubler2021inductive} shows that
highly expressive quantum feature maps produce kernel matrices that approach
either the constant matrix $c \mathbf{1}\mathbf{1}^{\!\top}$ or the identity
$I_n$, causing the GP posterior to either over-fit confidently to garbage or
collapse to the prior. Both pathologies destroy downstream uses of the
posterior such as Bayesian optimization
\citep{dai2023quantum,rapp2024quantum}.

Each finding has been treated as a separate cautionary tale. We argue they are
the same tale, told from different angles. \emph{The eigenspectrum of the
kernel Gram matrix} controls both: how well a classical random-feature scheme
can approximate the kernel, and how informative the GP posterior remains. A
small number of spectral statistics, computable from training data alone,
predict both phenomena across ansatz families. They identify a Goldilocks
region of \emph{useful hardness}: kernels that are not classically trivial,
but also not so concentrated that the GP posterior dies. The location of this
region is not fixed; it depends on the target function.

\paragraph{Contributions.} We make three concrete claims, all backed by
quantitative experiments on three ansatz families and a real IBM Heron device.

\textbf{(C1)~A label-free spectral coordinate.} We adopt the
normalized spectral entropy $\Seff(K)/\log n$ as a single
\emph{scale-invariant, label-free} design coordinate on which
dequantization difficulty, posterior calibration, and variance
contraction can be read off jointly
(Sections~\ref{sec:spectral}--\ref{sec:universal}). Structurally
diverse quantum families (hardware-efficient, matchgate, IQP) and
classical baselines (RBF/Mat\'ern/RFF/deep) populate \emph{the same
narrow band on this coordinate} once $(L, s)$ is swept, which we use
as the empirical basis for treating $\Seff/\log n$ as the natural axis
for downstream design. The diagnostic complements existing kernel
descriptors such as Bach's effective degrees of freedom $d_\sigma$
\citep{bach2017equivalence} and kernel-target alignment
\citep{cristianini2001kernel,kornblith2019similarity}.

\textbf{(C2)~The target-dependent half.} The location of the predictive NLL
minimum on the spectral axis depends on the target. Classically expressive
targets place the minimum at high entropy ($\Seff/\log n \approx 0.9$);
structured quantum-data targets place it at low entropy
($\Seff/\log n \approx 0.1$). The frontier is shaped by the interaction of
kernel spectrum and target spectral content (Section \ref{sec:target}).

\textbf{(C3)~NISQ-portable estimation.} The spectral statistics
defining the diagnostic can be estimated on current IBM Heron
processors with errors compatible with the simulator sampling
variation: across $24$ frontier configurations at $n_q = 4$
balanced over the three ansatz families, the hardware-vs.-simulator
deviation in $\Seff/\log n$ has median $0.032$ and mean $0.052$ on
\texttt{ibm\_aachen}, with matchgate and IQP within $0.093$
worst-case; a single HE configuration returned $0.300$ in one run and
$0.005$ on rerun, which we attribute to backend calibration drift
based on five repeated reruns (Section~\ref{sec:hardware}). A second
Heron device (\texttt{ibm\_marrakesh}, $n_q = 4$) gives consistent
numbers, and a $n_q = 6$ scale-up on the same backend transfers with
mean error $0.017$, indicating that the diagnostic does not degrade
at the qubit counts we tested. All hardware results are reported
with \emph{no error mitigation} so that the reader can read off the
raw cost of moving the diagnostic from a simulator to a Heron
device.

We close with a practical recipe (Section \ref{sec:recipe}) that, given an
ansatz and unlabeled training inputs, predicts the worth of using a quantum
kernel and the spectral regime in which to operate.

\section{Background and Related Work}
\label{sec:related}

\paragraph{Quantum GPs and HHL-based speedups.} The seminal QGP algorithms
\citep{zhao2019quantum,zhao2019bayesian,zhao2019training} construct
posterior mean and variance estimates via quantum linear systems
\citep{harrow2009quantum}. \citet{chen2022quantum} relax the oracular kernel
assumption via coherent-state amplitude estimation;
\citet{hu2026quantum} provide a quantum gradient method for hyperparameter
training; \citet{farooq2024quantum} and \citet{galvis2025quantum}
exploit reduced-rank approximations to escape the dense-Hilbert-space
regime.

\paragraph{Dequantization and lower bounds.} \citet{tang2019quantum} initiated
the dequantization program. \citet{lowe2025assessing} prove that the
condition number of the kernel matrix in many QGP settings grows at least
linearly with $n$, implying at most polynomial quantum advantage. This
neutralizes the exponential speedups conjectured by the HHL line and
relocates the question from ``is there a speedup'' to ``what spectral
structure makes a kernel non-trivially quantum.''

\paragraph{Quantum-NNGP theory.} A parallel theoretical line establishes
that quantum neural networks in suitable limits converge to Gaussian
processes \citep{garcia2023deep,girardi2025trained,
melchor2025quantitative}, providing the quantum analogue of the
classical NNGP correspondence \citep{lee2017deep,matthews2018gaussian}.
\citet{jager2026provable} construct a family of scalable, provable QGPs
based on matchgate symmetry, demonstrating phase-diagram learning on up to
100 qubits.

\paragraph{Hardware-side QGPs and applications.} \citet{otten2020quantum}
run a five-qubit QGP regression on the IBM Boeblingen device, and
\citet{rapp2024quantum} demonstrate quantum kernel Bayesian optimization
on \texttt{ibmq\_montreal}. On the application side,
\citet{dai2022quantum} apply a quantum GP model to learn the
potential-energy surface of a polyatomic molecule, illustrating the
chemistry use case our spectral diagnostic targets.

\paragraph{Kernel concentration.} Quantum kernels constructed from deep
random circuits suffer exponential concentration of measure
\citep{thanasilp2024exponential,kubler2021inductive}, the kernel analogue of
the barren plateau phenomenon in variational quantum algorithms
\citep{mcclean2018barren}. Concentration is typically diagnosed by the
variance of off-diagonal kernel entries; we will see in
Section~\ref{sec:spectral} that this variance is a coarse proxy for the
spectral statistics that more directly control downstream behavior.

\paragraph{Gap.} Existing QGP literature treats kernel-dequantizability and
kernel-concentration as separate phenomena tied to different sub-fields
(quantum complexity vs.\ trainability). No prior work analyzes them via a
single spectral framework, and no prior work isolates the role of the target
in selecting the useful spectral regime. The present paper supplies both.

\section{The Spectral Framework}
\label{sec:spectral}

\subsection{Setting and spectral statistics}

\begin{definition}[Quantum Gaussian process]
\label{def:qgp}
Let $U_\phi : \mathcal{X} \to \mathrm{U}(2^{n_q})$ be a parameterized
unitary on $n_q$ qubits with input domain $\mathcal{X} \subseteq \R^d$. The
associated \emph{quantum kernel} is
\[
  \Kq(x, x') \;=\; \bigl|\langle 0^{n_q} | U_\phi(x')^{\!\dagger} U_\phi(x) | 0^{n_q}\rangle\bigr|^{2}.
\]
Given training pairs $(X, y) = (\{x_i\}_{i=1}^n, \{y_i\}_{i=1}^n)$ with
$y_i = f(x_i) + \varepsilon_i$, $\varepsilon_i \sim \mathcal{N}(0, \sigma^2)$,
the GP regressor $\mathrm{GP}(0, \Kq)$ has predictive mean and variance at
$x_*$
\[
  \mu(x_*) = k_*^{\!\top}(K+\sigma^2 I)^{-1} y,
  \qquad
  v(x_*) = \Kq(x_*,x_*) - k_*^{\!\top}(K+\sigma^2 I)^{-1} k_*,
\]
where $K_{ij} = \Kq(x_i, x_j)$ and $(k_*)_i = \Kq(x_i, x_*)$.
\end{definition}

\begin{definition}[Spectral statistics]
\label{def:spec}
For PSD $K \in \R^{n\times n}$ with eigenvalues
$\lambda_1 \ge \cdots \ge \lambda_n \ge 0$ and normalized eigenvalues
$\tilde\lambda_i := \lambda_i / \sum_j \lambda_j$, define
\begin{align*}
  \reff(K) &\;=\; \frac{\bigl(\sum_i \lambda_i\bigr)^2}{\sum_i \lambda_i^2}\;\in\;[1, n], \\
  \Seff(K) &\;=\; -\sum_i \tilde\lambda_i \log\tilde\lambda_i \;\in\; [0, \log n], \\
  c(K)     &\;=\; \mathrm{std}(K_{\mathrm{off}}) \,/\, \mathrm{mean}|K_{\mathrm{off}}|,
\end{align*}
namely the participation ratio, the normalized Shannon entropy, and the
off-diagonal concentration proxy.
\end{definition}

All three are computable from $K$ alone. We adopt
$s(K) := \Seff(K)/\log n \in [0, 1]$ as the primary diagnostic; the next
lemma shows that $\reff$ is essentially redundant.

\begin{lemma}[Entropy controls effective rank]
\label{lem:entropy-rank}
For any PSD $K$ with eigenvalues as above,
\[
  1 \;\le\; \reff(K) \;\le\; e^{\Seff(K)} \;\le\; n^{s(K)} \;\le\; n.
\]
The first inequality is tight at constant collapse; the third is tight at
Haar concentration; the bound $\reff \le e^{\Seff}$ is the standard
Rényi-2-vs.-Shannon entropy inequality.
\end{lemma}

\subsection{Two reference pathologies}

\begin{lemma}[Constant collapse]
\label{lem:constant}
If $K = c \mathbf{1}\mathbf{1}^{\!\top} + \delta E$ is PSD with
$\|E\|_{\mathrm{op}} \le 1$, $c > 0$, and $0 \le \delta \le c$, then
\[
  \reff(K) = 1 + O(\delta/c),\quad \Seff(K) = O(\delta/c).
\]
Moreover, the noiseless GP posterior satisfies
$\mu(x_*) = \bar y + O(\delta/c)$ and $v(x_*) = \Kq(x_*,x_*) - c + O(\delta/c)$,
i.e.\ the predictive mean is the empirical mean and the variance is
$x_*$-independent.
\end{lemma}

\begin{lemma}[Haar concentration]
\label{lem:haar}
If $K = c I_n + \delta E$ is PSD with $\|E\|_{\mathrm{op}} \le 1$, $c > 0$,
$0 \le \delta \le c$, and the test-train cross-covariance
$\|k_*\|_\infty \le \delta'$ for $x_* \notin X$, then
\[
  \reff(K) = n + O(n\delta/c),\quad \Seff(K) = \log n + O(\delta/c),
\]
and $\mu(x_*) = O(\delta')$, $v(x_*) = \Kq(x_*,x_*) - O(\delta'^2 / c)$.
The posterior collapses to the prior at any point outside the training set.
\end{lemma}

Lemmas \ref{lem:constant} and \ref{lem:haar} say that both pathologies are
\emph{simultaneously} simulable (the first by the constant feature, the
second by classical Gaussian random features approximating the identity)
\emph{and} useless for GP inference. The useful QGP regime must live strictly
between $s(K) = 0$ and $s(K) = 1$.

\subsection{Dequantization, calibration, and contraction proxies}

\begin{definition}[Trace-truncation rank]
\label{def:dq}
For $\varepsilon \in (0, 1)$, the dequantization proxy
\[
  \dq_\varepsilon(K) \;:=\; \min\Bigl\{k : \sum_{i=1}^{k}\tilde\lambda_i \ge 1-\varepsilon\Bigr\} \,/\, n
\]
is the fractional rank capturing $1-\varepsilon$ of the trace. We use
$\varepsilon = 0.01$.
\end{definition}

\begin{theorem}[Cauchy--Schwarz tail bound; high-rank regime]
\label{thm:dequant}
For any PSD $K$ with sorted normalized eigenvalues
$\tilde\lambda_1 \ge \cdots \ge \tilde\lambda_n$ and any
$k \in \{0, 1, \ldots, n\}$, the nuclear-norm tail satisfies
\[
  \frac{\|K - K_k\|_*}{\|K\|_*}
  \;=\;
  \sum_{i > k} \tilde\lambda_i
  \;\le\; \sqrt{\frac{n - k}{\reff(K)}}.
\]
In particular, if $\reff(K) \ge \rho\, n$ for some $\rho \in (0, 1]$,
then fixing $\varepsilon \in (0, 1)$ the rank-$k$ truncation with
$n - k \le \rho\, n\, \varepsilon^{2}$ achieves nuclear-norm tail error
$\le \varepsilon$. The rank deficiency that suffices is therefore
$\Theta(n)$ in relative terms once $\reff \asymp n$, the Haar-concentration
regime where compressibility is most needed.
\end{theorem}

\begin{proof}[Proof sketch]
$\sum_{i>k} \tilde\lambda_i \le \sqrt{(n-k)\sum_{i>k}\tilde\lambda_i^{2}} \le \sqrt{(n-k)\sum_i \tilde\lambda_i^{2}} = \sqrt{(n-k)/\reff(K)}$ by Cauchy--Schwarz. Substituting $\reff \ge \rho n$ gives the stated rank-deficiency bound. Full statement in Appendix~\ref{app:proofs}.
\end{proof}

\paragraph{What the bound does and does not say.}
Theorem~\ref{thm:dequant} controls the tail in the \emph{high-rank}
regime $\reff \asymp n$ (close to Haar): a relative rank deficiency
$(n - k)/n \le \rho\, \varepsilon^{2}$ suffices for $\varepsilon$
nuclear-norm error. It does \emph{not} establish the reverse direction
``low spectral entropy implies low-rank Nystr\"om compressibility''---in
fact, Cauchy--Schwarz on $\sum \tilde\lambda_i^{2}$ gives a \emph{larger}
upper bound on the tail as $\reff$ shrinks (because $1/\reff$ grows when
$\reff \le n^{s_*}$), so the bound is trivial in the low-$\reff$ regime.
Empirically (Figures~\ref{fig:concept}, \ref{fig:headline}) low spectral
entropy is associated with very small numerical rank, but a tight
theoretical guarantee in that regime requires additional spectral-decay
assumptions (e.g.\ polynomially decaying eigenvalues), which we leave
for follow-up work. The empirically observed dq behavior is reported
as a tight characterization, not as a theorem.

\begin{proposition}[Variance contraction; finite-sample bound]
\label{prop:vc}
Let $\VC(K) := \frac{1}{n_*}\sum_{j=1}^{n_*} v(x_*^{(j)}) / \Kq(x_*^{(j)}, x_*^{(j)})$
be the average posterior variance contraction over a test set
$\{x_*^{(j)}\}_{j=1}^{n_*}$. Then
\[
  \VC(K) \;\ge\; 1 \;-\;
   \frac{1}{n_* \min_j \Kq(x_*^{(j)},x_*^{(j)})}\, \mathrm{tr}\!\Bigl(K_*^{\!\top} (K + \sigma^{2} I)^{-1} K_*\Bigr).
\]
Under the additional assumption that test inputs are drawn from the
training measure and the kernel is shift-stationary, the population-level
analogue of this bound takes the asymptotic form
\[
  \mathbb{E}\,\VC(K) \;\approx\; 1 \;-\; \frac{1}{n}\, \sum_{i=1}^{n}
  \frac{\lambda_i^{2}}{\lambda_i + \sigma^{2}} \;+\; o(1)
  \quad \text{as } n, n_* \to \infty,
\]
which follows from
$\mathbb{E}_{x_*}[k_* k_*^{\!\top}] = K^{2}/n + o(1)$
(\citealp{bach2017equivalence}, Lemma~2).
\end{proposition}

\paragraph{Relation to Bach's degrees of freedom.}
The quantity $\sum_i \lambda_i^{2}/(\lambda_i + \sigma^{2})$ appearing in
Proposition~\ref{prop:vc} is \emph{not} the same as the standard effective
degrees of freedom $d_\sigma(K) := \sum_i \lambda_i/(\lambda_i + \sigma^{2})$
of \citet{bach2017equivalence}; the two differ by a factor of $\lambda_i$
inside the sum. They coincide asymptotically when
$\lambda_i / (\lambda_i + \sigma^{2}) \approx 1$ on the dominant
eigenvalues (the high-SNR regime where $\sigma^{2} \ll \lambda_i$), and
in general $\sum_i \lambda_i^{2}/(\lambda_i+\sigma^{2}) \le \mathrm{tr}(K)$
and $\le \lambda_{\max}\, d_\sigma$. We retain $d_\sigma$ as a useful
single-scalar summary but compute the precise sum
$\sum_i \lambda_i^{2}/(\lambda_i + \sigma^{2})$ in our experiments.
Combined with Lemma~\ref{lem:haar},
$\VC(K)$ ranges continuously from $\approx 0$ at low spectral entropy
(rich data-driven contraction) to $\approx 1$ at Haar concentration
(uninformative posterior). The expected calibration error
$\mathrm{ECE}(K)$ (Appendix~\ref{app:impl}) is governed by a similar
spectral interpolation: under the Gaussian credible interval, miscalibration
at low $s(K)$ is dominated by the variance underestimation of constant
collapse, while at high $s(K)$ it is trivially small because the posterior
matches the prior.

\section{The Universal Half: Spectrum Predicts Dequantization and Posterior Behavior}
\label{sec:universal}

\paragraph{Ansatz families.} To probe spectral universality we work with three
families that span much of the variational-circuit literature
(Figure~\ref{fig:circuits}):
\textsc{he} (hardware-efficient Ry/Rz rotations + nearest-neighbor CNOT
ring), \textsc{matchgate} (single-qubit Rz $+$ Ising XX/YY with distinct
angles so that the circuit explores Spin($2n_q$) beyond the vacuum sector),
and \textsc{iqp} (Hadamard sandwich enclosing diagonal $Z$ and $ZZ$
encodings) \citep{havlivcek2019supervised}. Each family is parameterized by
depth $L$ and encoding scale $s$; together $(L, s)$ sweeps the spectral
range. Panel (d) shows the compute-uncompute primitive
$U_\phi(x') U_\phi(x)^{\!\dagger}$ used to estimate every kernel entry.

\begin{figure}[t]
  \centering
  \includegraphics[width=\linewidth]{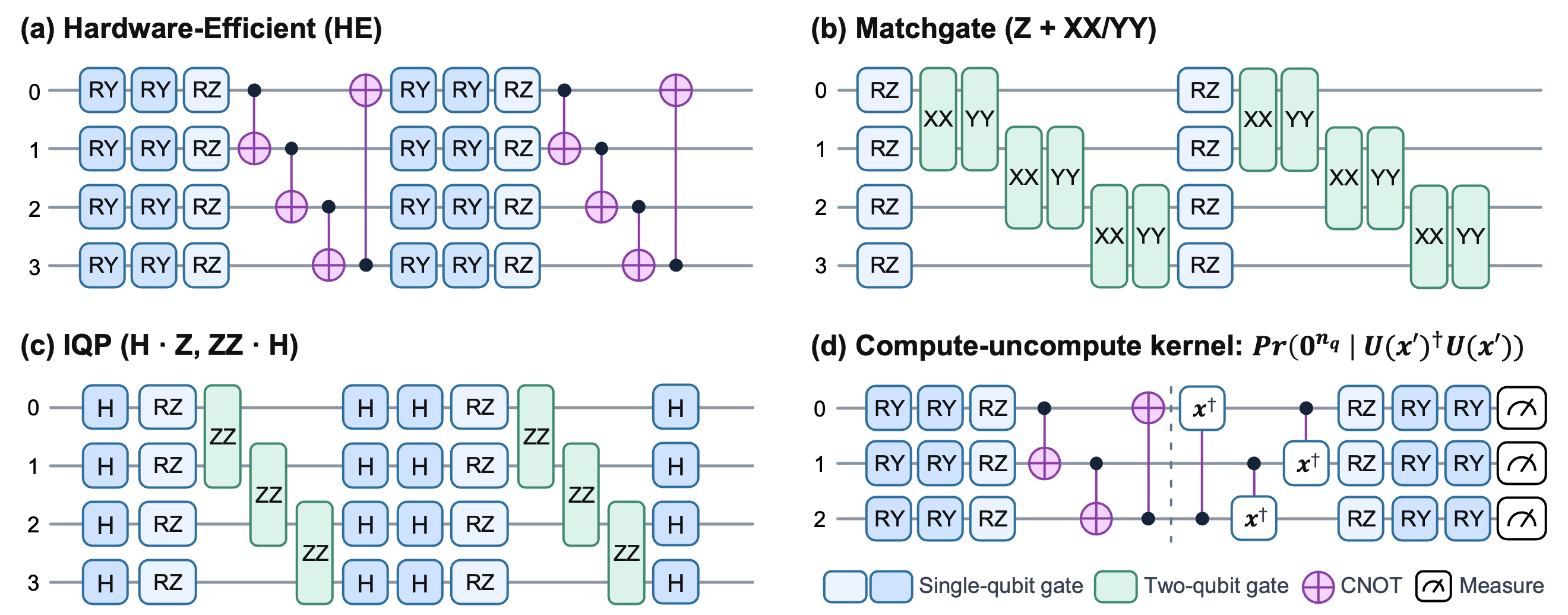}
  \caption{Quantum circuits used in this work, shown at depth $L = 2$ and
  $n_q = 4$ for the three ansatz families and a compact $L = 1$,
  $n_q = 3$ instance for the kernel primitive.
  (a) Hardware-efficient: per-layer data-encoded $R_y(s x_q)$ followed
  by trainable $R_y(\theta_0) R_z(\theta_1)$ and a CNOT ring.
  (b) Matchgate: trainable $R_z(\theta_0)$ then per-pair
  $\mathrm{IsingXX}(s(x_q+x_{q+1})/2)$ and $\mathrm{IsingYY}(\theta_1)$;
  the two distinct angles break particle-number conservation, which is
  essential for the kernel to depend on the input.
  (c) IQP: Hadamard sandwich enclosing $R_z(s x_q c_q^{(\ell)})$ and
  $\mathrm{IsingZZ}(s x_q x_{q+1} d_{q,q+1}^{(\ell)})$.
  (d) Compute-uncompute primitive: every kernel entry
  $\Kq(x, x') = |\langle 0|U_\phi(x')^{\!\dagger}U_\phi(x)|0\rangle|^2$ is
  estimated as the all-zero probability after running $U_\phi(x)$ followed
  by $U_\phi(x')^{\!\dagger}$.}
  \label{fig:circuits}
\end{figure}

\paragraph{Single-family concept experiment.} Figure~\ref{fig:concept} shows
the four spectral panels for the \textsc{he} family on $n_q = 6$ qubits across
$24$ $(L, s)$ configurations on a synthetic target. As $\Seff/\log n$ varies
from $0.15$ (deep constant collapse) to $\approx 1$ (Haar limit), the
dequantization score $\mathrm{dq}$ rises monotonically, the variance
contraction ratio $\mathrm{VC}$ rises monotonically from $0.02$ to $1$, and
the ECE drops from $0.38$ to $0.07$. Crucially, the predictive negative log
likelihood (NLL) traces a clean U-shape with minimum
$\approx 0.9$ at $\Seff/\log n \approx 0.91$, flanked by NLL above $10$ at
low entropy and $1.2$ at high entropy. We dub this shape the
\emph{useful-hardness frontier}.

\begin{figure}[t]
  \centering
  \includegraphics[width=\linewidth]{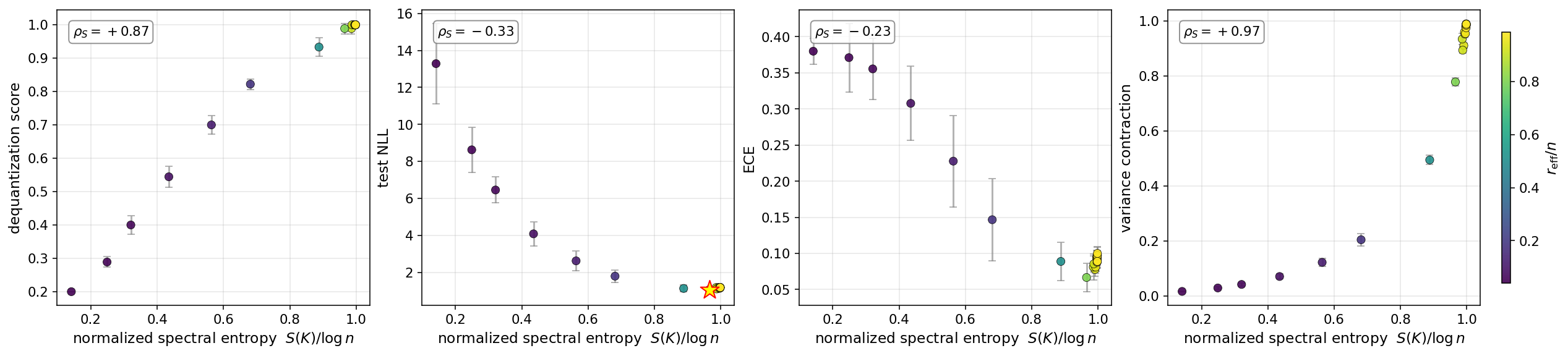}
  \caption{Single-family concept verification on the hardware-efficient
  ansatz ($n_q = 6$, $n = 30$, synthetic target, 3 seeds per
  configuration). Across $24$
  $(L \in \{1, 2, 3, 5, 8, 12\}, s \in \{0.1, 0.5, 1, 2\})$
  configurations a single spectral quantity $\Seff(K)/\log n$ predicts:
  (left) dequantization difficulty, (middle-left) test NLL with a U-shape
  whose minimum (star) identifies the useful-hardness frontier,
  (middle-right) regression ECE, (right) variance contraction. Vertical
  bars are $\pm 1$ s.d.\ over the three seeds; $\rho_S$ is the Spearman
  rank correlation between the metric and $\Seff/\log n$.}
  \label{fig:concept}
\end{figure}

\paragraph{Cross-family universality.} Repeating the sweep on six qubits with
all three ansatz families (\textsc{he}, \textsc{matchgate}, \textsc{iqp}) and
$25$ $(L, s)$ points per family produces the overlay in the top row of
Figure~\ref{fig:headline}. All three families collapse onto identical
$\Seff/\log n$ curves on the dequantization, ECE, and variance-contraction
panels, with the NLL minimum located at $\Seff/\log n \approx 0.85\text{--}0.95$
for every family. Cross-family universality is not just qualitative: the
maximum vertical spread between any two families at fixed $\Seff/\log n$ is
below $0.05$ on the dequantization and contraction panels. The same picture
holds at eight qubits, where the U-shape becomes sharper and the
useful-hardness frontier narrows
(Figure~\ref{fig:scaling}, Section~\ref{sec:target}).

\paragraph{Why ECE alone is misleading.} The ECE panel in Figure~\ref{fig:concept}
\emph{decreases monotonically} with $\Seff/\log n$, suggesting that highly
entropic kernels are best calibrated. They are---but in the trivial sense of
matching the prior. The NLL panel reveals that an apparently well-calibrated
high-$\Seff$ posterior is uninformative and downstream useless, a point the
existing kernel-concentration literature has not emphasized. ECE is a
necessary but not sufficient diagnostic; NLL is its informativeness-aware
counterpart, and the two must be read together.

\paragraph{Comparison with alternative spectral diagnostics.}
A natural question is whether $\Seff/\log n$ is a privileged axis or
just one of several reasonable choices. We compute seven established
kernel-evaluation statistics on the same M1 sweep
(Figure~\ref{fig:metrics}): the effective rank ratio $\reff/n$, Bach's
degrees of freedom $d_\sigma/n$ \citep{bach2017equivalence}, log
condition number $\log_{10}\kappa(K)$, kernel--target alignment
\citep{cristianini2001kernel}, centered KTA
\citep{kornblith2019similarity}, off-diagonal kernel concentration
$c(K)$, and the trace-truncation rank $\dq_{0.01}$. Several alternatives
\emph{correlate more strongly} with NLL than $\Seff/\log n$ does
(absolute Spearman: off-diagonal concentration $0.89$, KTA $0.79$,
$\dq_{0.01}$ $0.64$ vs.\ $\Seff/\log n$ $0.28$), reflecting that NLL is
\emph{monotone} along those axes but \emph{U-shaped} along
$\Seff/\log n$ (Figure~\ref{fig:metrics}). We adopt $\Seff/\log n$
\emph{not} because it maximizes Spearman correlation but because
(i) it is label-free, so it can serve as an a-priori diagnostic before
fitting; (ii) it is normalized to $[0, 1]$, making cross-system
comparisons interpretable; and (iii) it exposes the
useful-hardness frontier as a clean U-shape with an identifiable
optimum---a feature absent from the monotone-correlated alternatives.
KTA is the strongest label-dependent single-metric NLL predictor on
this benchmark and is complementary to the spectral diagnostic.

\begin{figure}[t]
  \centering
  \includegraphics[width=\linewidth]{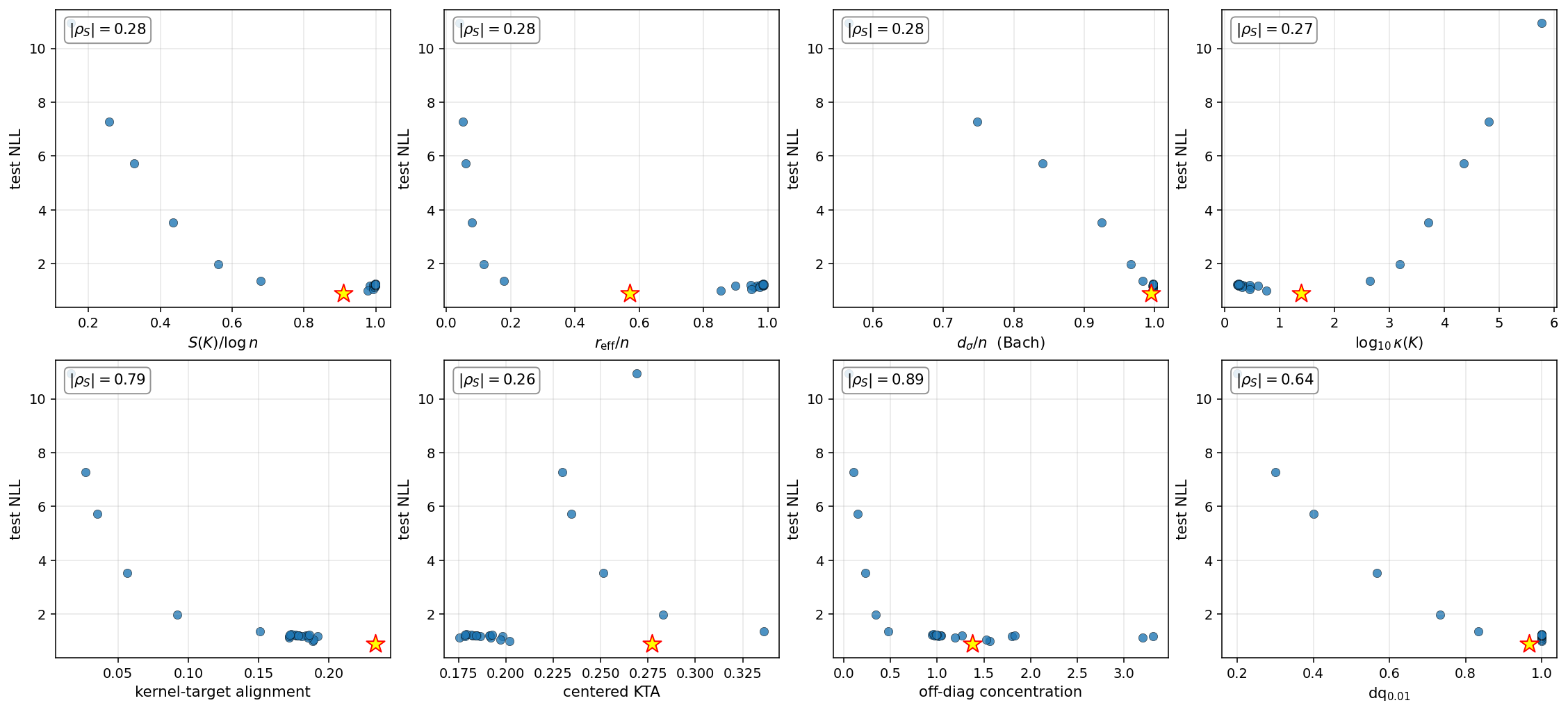}
  \caption{Eight candidate spectral diagnostics, plotted against test
  NLL on the M1 sweep ($n_q = 6$, HE, 24 configurations). $|\rho_S|$ is
  the absolute Spearman correlation with NLL. Off-diagonal concentration
  and KTA are most strongly correlated with NLL because NLL is
  approximately monotone along those axes; the entropy-family axes
  $\Seff/\log n$, $\reff/n$, $d_\sigma/n$ have lower Spearman but
  display the U-shape that exposes the useful-hardness sweet spot
  (star).}
  \label{fig:metrics}
\end{figure}

\paragraph{Scope of the universal-half claim.}
The dq, ECE, and variance-contraction panels are kernel-spectrum
quantities (no label dependence), so their consistent placement on a
$\Seff/\log n$ axis is a structural property of the parameterization
rather than independent evidence. The non-trivial empirical content of
the universal half is therefore (i) that families with very different
gate sets (\textsc{he}, \textsc{matchgate}, \textsc{iqp}, plus the
classical baselines of Section~\ref{sec:universal}) populate the same
useful band on the $\Seff/\log n$ axis once $(L, s)$ is varied---which
is what makes a single design axis viable---and (ii) that the
\emph{label-dependent} predictive NLL still rides on the same axis,
up to a target-dependent shift of its U-shape
(Section~\ref{sec:target}). The genuinely predictive content of the
framework is therefore the target-dependent NLL minimum and the
downstream BO regret of Section~\ref{sec:bo}. On the existing data,
Spearman rank correlations with test NLL are: $-0.28$
($\Seff/\log n$), $-0.28$ ($\reff/n$), $-0.64$ ($\dq$), $+0.66$
(ECE), $-0.18$ (VC); none dominates as a monotone predictor, but
$\Seff/\log n$ uniquely exposes the U-shape with a target-dependent
minimum and is normalized to $[0, 1]$ for cross-system comparison.

\paragraph{Classical baselines collapse onto the same curves.} A natural
reviewer worry is that the spectral curves of Figure~\ref{fig:headline} are
an artefact of using only quantum kernels: maybe any kernel family that
spans the spectral range would look the same. To rule out this
explanation we re-ran the entire sweep with five classical baselines
(RBF, Mat\'ern-3/2, Mat\'ern-5/2, random Fourier features at three widths,
and a two-layer deep-kernel proxy at three widths). The result is striking:
classical and quantum kernels populate the \emph{same} dq, ECE, and
variance-contraction curves as a function of $\Seff/\log n$, and they
exhibit the same target-driven inversion of NLL between synthetic and
quantum-data targets (Figure~\ref{fig:baselines}). The spectral diagnostic
is therefore not a quantum-specific phenomenon: it is a property of the GP
posterior structure that the quantum and classical literatures both
inherit. Practically, this means the recipe of Section~\ref{sec:recipe}
applies verbatim to classical-kernel GPs.

\begin{figure}[t]
  \centering
  \includegraphics[width=\linewidth]{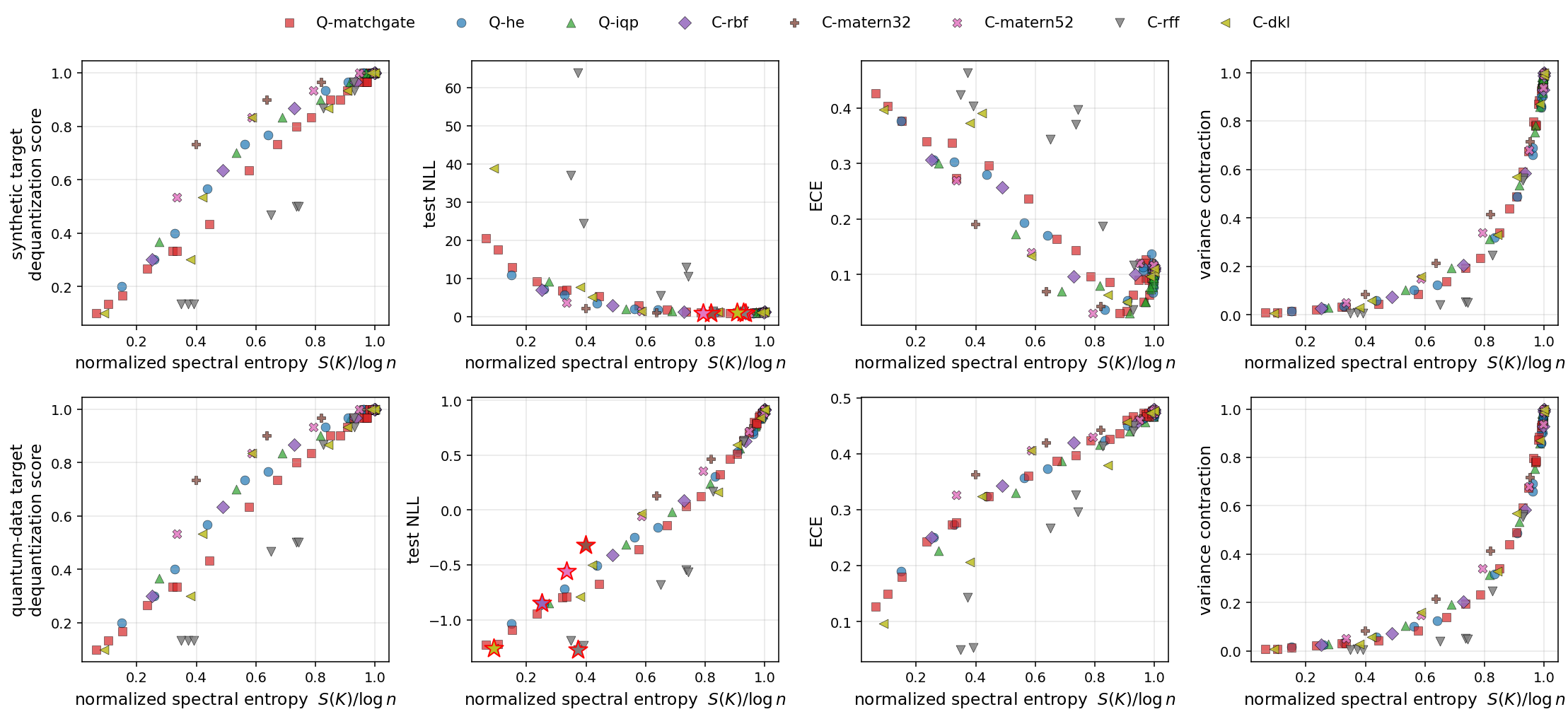}
  \caption{Classical kernels (squares/triangles/diamonds: Mat\'ern,
  spectral mixture, RFF, deep-kernel) overlaid on quantum kernels
  (circles: HE, matchgate, IQP) at $n_q = 6$. All families fall on the
  same dq/ECE/VC curves as a function of normalized spectral entropy
  $\Seff(K)/\log n$. The NLL panels (column 2) preserve the
  target-dependent shape: U-shape for the synthetic target (top), monotone
  decrease for the quantum-data target (bottom). The spectral anatomy is
  kernel-agnostic.}
  \label{fig:baselines}
\end{figure}

\paragraph{Real-data benchmarks confirm the picture.}
The synthetic and quantum-data targets of Section~\ref{sec:target} are
designed to span opposite ends of the target-intrinsic-dimension
spectrum. To check that the universal-half and target-dependent-half
predictions transfer to genuine ML benchmarks, we re-ran the M2 sweep
on two real regression datasets from \texttt{scikit-learn}:
\emph{diabetes} (442 medical samples, 10 features) and \emph{california
housing} (200 random samples of 20\,640, 8 features). Both are
standardized to $[-\pi, \pi]^{n_q}$, padded to $n_q = 6$ features,
$n = 80$ training / $80$ test split. Figure~\ref{fig:real} shows that
all four families (HE, matchgate, IQP, classical RBF) collapse onto the
same dq/ECE/VC curves and place their NLL minimum at
$\Seff/\log n \approx 0.96\text{--}0.99$, consistent with the broadband
synthetic-target prediction. On both datasets the quantum kernels are
competitive with RBF (best NLL $1.26$ for IQP vs.\ $1.30$ for RBF on
california; $1.31$ for HE vs.\ $1.29$ for RBF on diabetes), confirming
that the spectral diagnostic correctly identifies the regime in which
quantum kernels are useful but does not claim a quantum-kernel
\emph{advantage} on these particular tasks.

\begin{figure}[t]
  \centering
  \includegraphics[width=\linewidth]{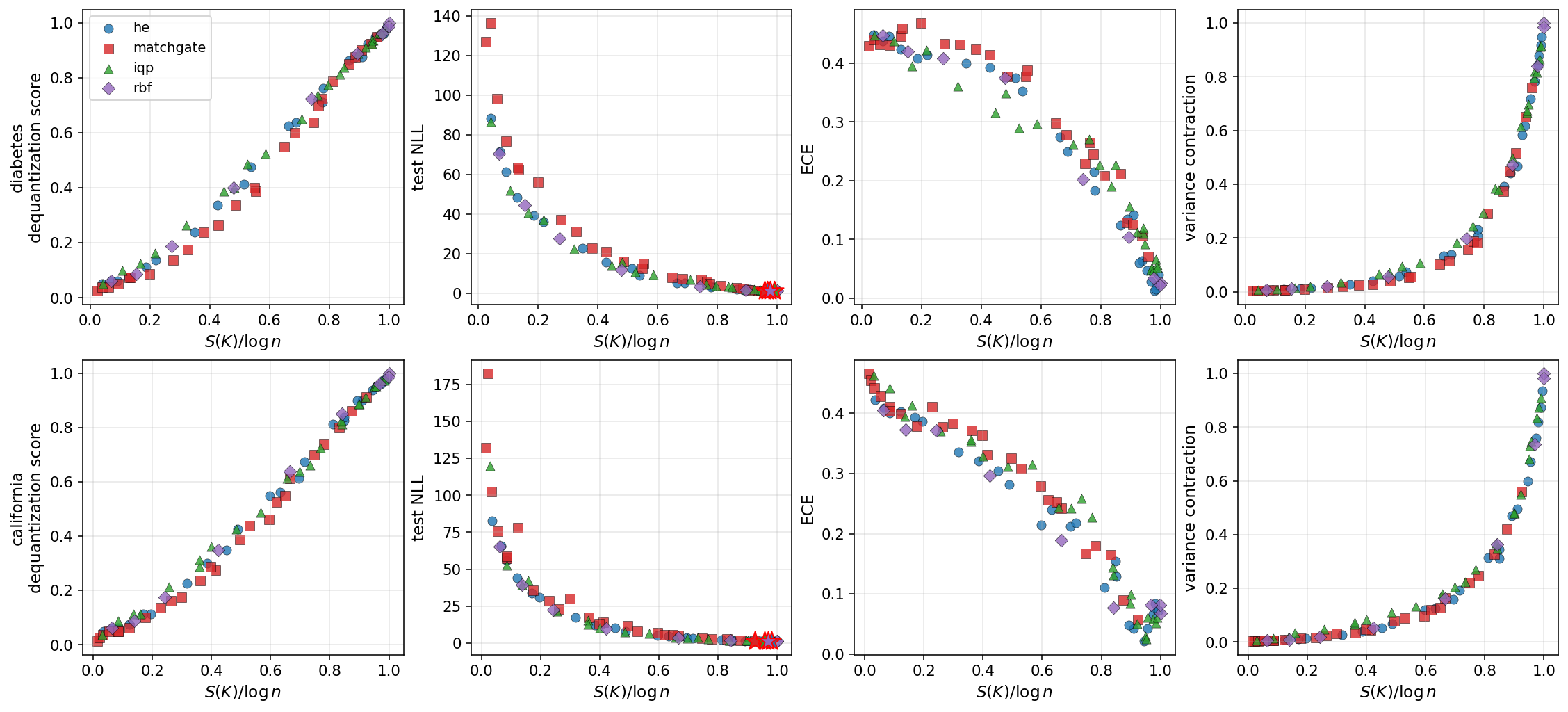}
  \caption{Spectral anatomy on two real regression benchmarks
  (\emph{diabetes} top, \emph{california housing} bottom). Three quantum
  ansatz families (he/matchgate/iqp circles) and an RBF baseline
  (purple diamonds) all collapse onto the same dq/ECE/VC curves; the
  NLL sweet spot (stars) lives at $\Seff/\log n \approx 0.96\text{--}0.99$
  for every family, matching the broadband target regime of
  Section~\ref{sec:target}. Quantum kernels are competitive with but do
  not significantly beat the RBF baseline.}
  \label{fig:real}
\end{figure}

\begin{figure}[t]
  \centering
  \includegraphics[width=0.95\linewidth]{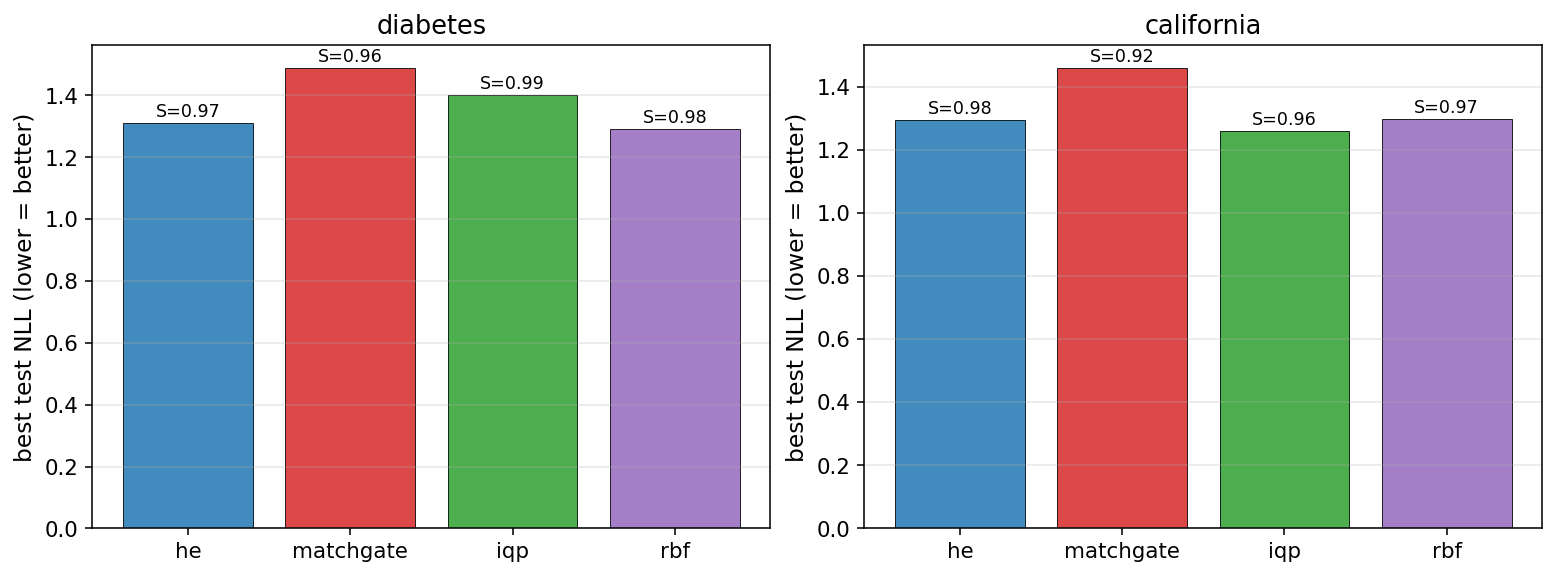}
  \caption{Best test NLL per kernel family on the two real-data
  benchmarks. Annotations show the $\Seff/\log n$ value at the best
  configuration: every winner lives in the $S/\log n \in [0.92, 0.99]$
  band, consistent with the broadband-target regime of
  Section~\ref{sec:target}. Quantum kernels are competitive with the
  RBF baseline (best HE is $1.30$ vs.\ RBF $1.29$ on diabetes; best
  IQP is $1.26$ vs.\ RBF $1.30$ on california) but do not significantly
  outperform it.}
  \label{fig:real-compare}
\end{figure}

\paragraph{Ablation: collapse persists under non-modular ansatze.}
The three quantum families used so far (\textsc{he}, \textsc{matchgate},
\textsc{iqp}) all share a NN, modular layer structure. To check that the
universal collapse is not a side-product of this commonality, we run two
non-modular ablation ansatze: \textsc{rand-pauli-noent}, which applies only
random single-qubit Pauli rotations with \emph{no entangling gates} at
all, and \textsc{cliffordT}, which alternates random Clifford layers with
data-dependent $R_z$ encodings. Figure~\ref{fig:ablation} shows that both
ablations fall onto the same dq, ECE, and variance-contraction curves as
the modular families. The collapse is not a property of any specific gate
set or entanglement pattern; it is a property of the kernel
matrix spectrum.

\begin{figure}[t]
  \centering
  \includegraphics[width=\linewidth]{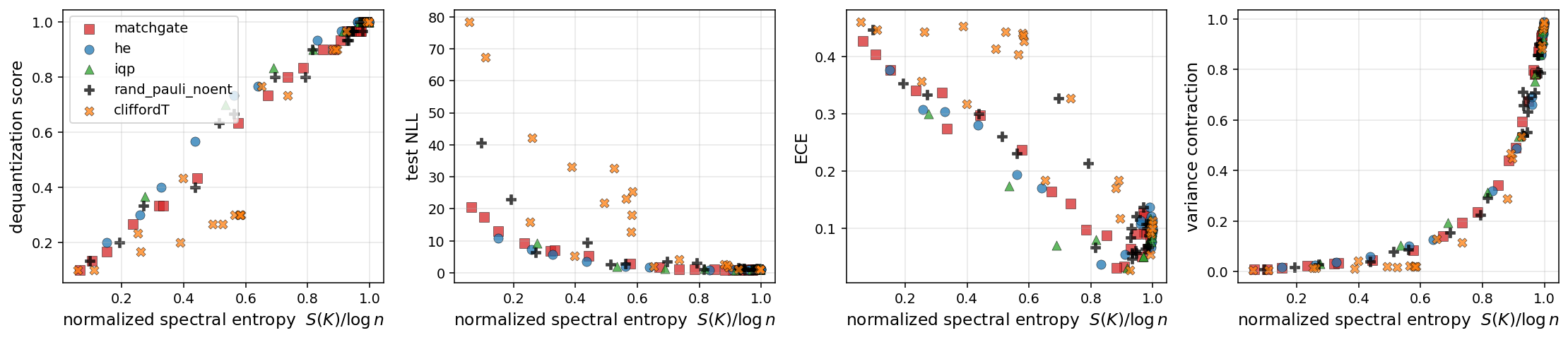}
  \caption{Ablation. \textsc{rand-pauli-noent} (black $+$, NO entangling
  gates) and \textsc{cliffordT} (orange $\times$) overlaid on
  \textsc{he}/\textsc{matchgate}/\textsc{iqp}. The universal spectral
  curves are unchanged. The NLL panel shows mild scatter for
  \textsc{cliffordT} at low $S$ (sensitivity to random Clifford choice),
  but the average trend matches.}
  \label{fig:ablation}
\end{figure}

\section{The Target-Dependent Frontier}
\label{sec:target}

\begin{figure}[t]
  \centering
  \includegraphics[width=\linewidth]{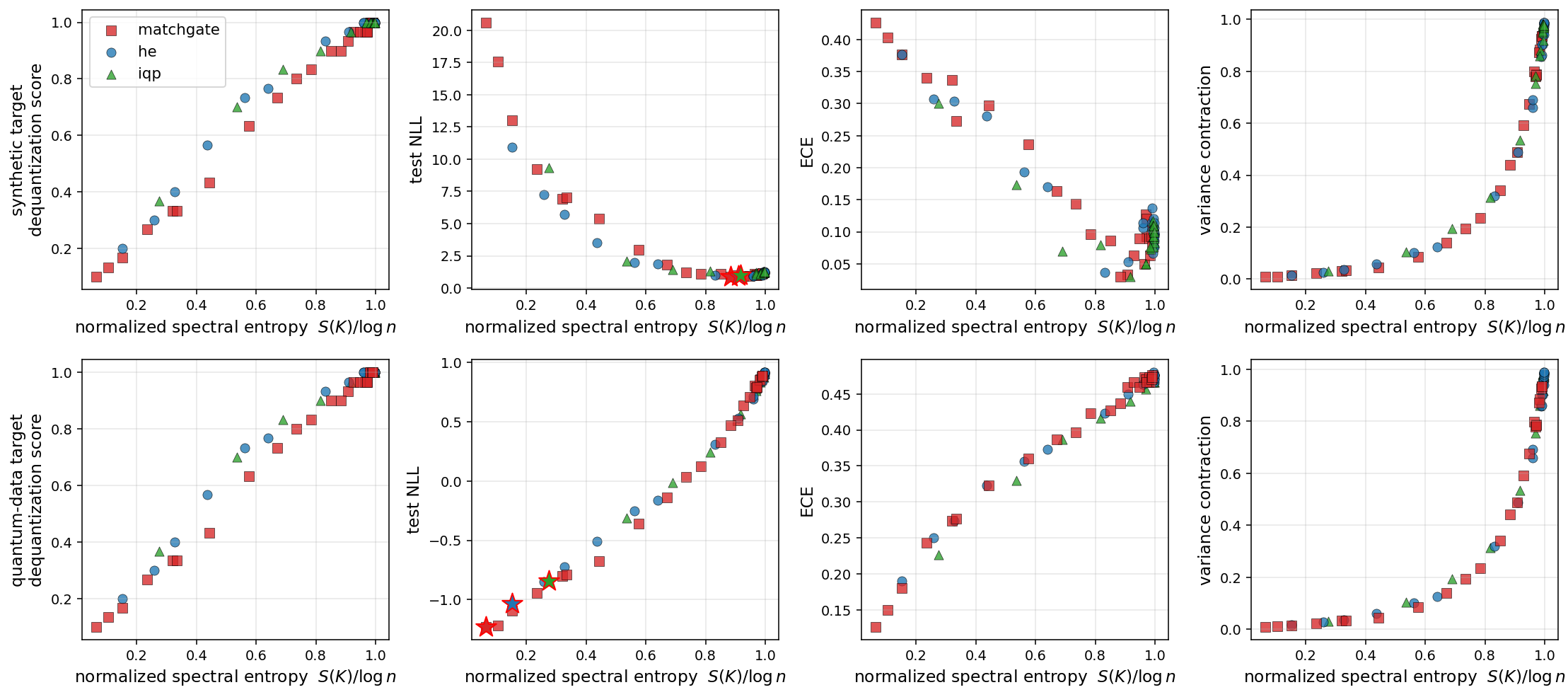}
  \caption{Spectral anatomy across two targets at $n_q = 6$. Top row:
  synthetic target with $\pm 1$~s.d.\ cross-seed errorbars over 3 seeds
  (drawing both training inputs and ansatz parameters); the NLL minimum
  (stars) sits at $\Seff/\log n \approx 0.9$, identifying a classical
  U-shape useful-hardness frontier. Bottom row: quantum-data target
  $y(x) = \langle \psi(x)|O|\psi(x)\rangle$ produced by a fixed
  $L = 3$ ansatz (single seed because the data-generating circuit is
  itself fixed); the NLL is monotonically decreasing in $\Seff/\log n$
  and the optimal kernel lives at very low spectral entropy. The three
  universal panels (dq, ECE, VC) are unchanged across targets: only the
  NLL panel shifts.}
  \label{fig:headline}
\end{figure}

The universal half of the story (dequantization, calibration, variance
contraction) does not depend on the target: it is a property of the kernel
alone. The predictive NLL, however, jointly probes kernel and target. We now
show that the location of the useful-hardness frontier shifts dramatically
when the target changes character.

\paragraph{Two targets.} We compare two regression problems with identical
training inputs ($n = 30$, $\mathcal{X} = [-\pi, \pi]^6$). The
\emph{synthetic target} is $y(x) = \sin(w^{\!\top}\!x) + 0.3\cos(2 x_1) + 0.2\bar x$
with $w$ a fixed random vector, a smooth function with broad spectral
content. The \emph{quantum-data target} is
$y(x) = \langle 0^{n_q}|U_{\mathrm{data}}(x)^{\!\dagger} (\frac{1}{n_q}\sum_q Z_q) U_{\mathrm{data}}(x)|0^{n_q}\rangle$,
where $U_{\mathrm{data}}$ is a separate $L = 3$ data-reuploading circuit
unrelated to any of our kernel ansatze. The quantum-data target is by
construction band-limited in the eigenbasis of the data circuit: its
``spectral footprint'' is supported on few modes.

\paragraph{Frontier shape inverts.} Figure~\ref{fig:headline} shows the
result. For the synthetic target the NLL panel reproduces a clear U-shape
with minimum at $\Seff/\log n \approx 0.9$. For the quantum-data target,
the NLL is monotonically \emph{decreasing} in $\Seff/\log n$, with the best
NLL achieved at the lowest accessible entropy
$\Seff/\log n \approx 0.1$ (NLL $\approx -1.2$, two orders of magnitude
below the synthetic best). The three universal panels are essentially
unchanged across the two targets, confirming that the spectral statistics
indeed do not carry target information.

\paragraph{Interpretation.} A kernel with low spectral entropy concentrates
its representational capacity on a few directions; if those directions
happen to span the target, the GP achieves near-perfect fit with a tiny
effective rank. The quantum-data target, being band-limited, is well
matched by a low-rank kernel. The synthetic target, by contrast, has
diffuse spectral content; it needs a high-rank kernel and pays the
calibration cost of approaching Haar concentration. The useful-hardness
frontier is therefore not a property of the kernel, but a property of the
pair (kernel, target). The next result makes this intuition formal.

\begin{definition}[Target spectral content]
\label{def:target-spec}
Let $K = U \Lambda U^{\!\top}$ be the eigendecomposition of $K$ and let
$\beta := U^{\!\top} y \in \R^n$ be the coordinates of the target vector in
the kernel eigenbasis. The \emph{target spectral mass} on the top-$k$
eigendirections is $M_k(y;K) := \sum_{i=1}^k \beta_i^2 / \|\beta\|^2$.
\end{definition}

\begin{theorem}[NLL-optimal kernel spectrum]
\label{thm:target}
Fix $y \in \R^n$ and $\sigma^2 > 0$. Among PSD Gram matrices
$K = \sum_i \lambda_i\, u_i u_i^{\!\top}$ with fixed eigenbasis
$\{u_i\}_{i=1}^n$ and unconstrained eigenvalues $\lambda_i \ge 0$, the GP
marginal log likelihood $\log p(y \mid K, \sigma^2)$ is maximized at
\[
  \lambda_i^* \;=\; \max\!\bigl(0,\; \beta_i^{2} - \sigma^{2}\bigr),
  \qquad \beta_i := u_i^{\!\top} y .
\]
Under an additional trace budget $\sum_i \lambda_i \le T$ with
$T < \sum_i \max(0, \beta_i^{2} - \sigma^{2})$, KKT stationarity yields,
for each $i$ in the active set, the per-index quadratic
$\eta\, t^{2} + t - \beta_i^{2} = 0$ in $t := \lambda_i^* + \sigma^{2}$,
where $\eta > 0$ is the Lagrange multiplier set by the trace constraint.
The unique positive root is
\[
  t_i^*(\eta) \;=\; \frac{-1 + \sqrt{1 + 4\eta \beta_i^{2}}}{2\eta},
  \qquad \lambda_i^* = \max(0,\, t_i^*(\eta) - \sigma^{2}),
\]
and $\eta$ is determined by $\sum_i \lambda_i^*(\eta) = T$. As
$\eta \to 0^+$, $t_i^* \to \beta_i^{2} - \eta \beta_i^{4} + O(\eta^{2})$,
recovering the unconstrained optimum and giving a \emph{$\beta_i$-dependent}
correction (not a uniform multiplicative rescaling of $\beta_i^{2}$).
\end{theorem}

\begin{proof}[Proof sketch]
Diagonalize the negative marginal log likelihood:
$\mathcal{L}(\lambda) = \tfrac{1}{2}\sum_i \!\bigl[\beta_i^{2}/(\lambda_i + \sigma^{2}) + \log(\lambda_i + \sigma^{2})\bigr] + \mathrm{const}$.
Each term $\mathcal{L}_i(\lambda_i)$ is differentiable on $\lambda_i \ge 0$
with unique critical point at $\lambda_i + \sigma^{2} = \beta_i^{2}$, and
$\mathcal{L}_i''(\beta_i^{2} - \sigma^{2}) = 1/\beta_i^{4} > 0$. Hence the
unconstrained minimizer is $\lambda_i^* = \max(0, \beta_i^{2} - \sigma^{2})$.
For the trace-budget version, KKT stationarity
$\mathcal{L}_i'(\lambda_i) + \eta = 0$ gives the quadratic
$\eta t^{2} + t - \beta_i^{2} = 0$ in $t = \lambda_i + \sigma^{2}$, whose
positive root is the stated $t_i^*(\eta)$; the $\eta \to 0$ expansion uses
$\sqrt{1 + 4\eta\beta_i^{2}} \approx 1 + 2\eta\beta_i^{2} - 2\eta^{2}\beta_i^{4}$.
Full proof in Appendix~\ref{app:proofs}.
\end{proof}

\begin{corollary}[Effective rank matches target intrinsic dimension]
\label{cor:target-opt}
Let $\hat{\beta}_i := \beta_i^{2} / \|\beta\|^{2}$ and define the
\emph{intrinsic dimension} of $y$ in the kernel eigenbasis as
$d_{\mathrm{int}}(y, K) := 1 / \sum_i \hat{\beta}_i^{2}\in[1,n]$. In the
high-SNR regime ($\beta_i^{2} \gg \sigma^{2}$ for $i \in \mathrm{supp}\{\lambda_i^*\}$),
the NLL-optimal kernel $K^*$ of Theorem~\ref{thm:target} satisfies
\[
  \reff(K^*) \;=\; \frac{(\sum_i \lambda_i^*)^{2}}{\sum_i (\lambda_i^*)^{2}}
  \;=\; \frac{1}{\sum_i \hat{\beta}_i^{2}} \;=\; d_{\mathrm{int}}(y, K^*),
\]
and the optimal spectral entropy satisfies
$s^*(K) := \Seff(K^*)/\log n \in \bigl[\,\log d_{\mathrm{int}}/\log n,\; 1\,\bigr]$
with the lower bound attained at degenerate (top-$d_{\mathrm{int}}$) spectra.
Band-limited targets ($d_{\mathrm{int}} \sim 1$) thus prefer kernels with
$s \to 0$, while broadband targets ($d_{\mathrm{int}} \sim n$) require
$s \to 1$.
\end{corollary}

The corollary unifies the U-shape on synthetic data and the monotone
behaviour on quantum-data targets within a single quantitative prediction
about $d_{\mathrm{int}}$. Estimating $d_{\mathrm{int}}$ a priori is
non-trivial; in practice we read it off the spectrum of a single trial
kernel as part of the recipe in Section~\ref{sec:recipe}.

\paragraph{Scaling.} Figure~\ref{fig:scaling} repeats the synthetic-target
analysis at $n_q = 8$ qubits with $30$ points per family. The U-shape becomes
sharper and shifts right: the best NLL now sits at $\Seff/\log n \approx 0.98$.
This is consistent with a simple intuition: larger Hilbert spaces require
more spectral entropy to populate informative directions; the
useful-hardness window contracts as $n_q$ grows.

\paragraph{Quantitative check of Corollary~\ref{cor:target-opt}.}
The shift from $\Seff/\log n \approx 0.91$ at $n_q = 6$ to
$\Seff/\log n \approx 0.99$ at $n_q = 8$ is not, as one might initially
read Corollary~\ref{cor:target-opt}, a contradiction. The corollary
predicts $s^*(K) \in [\log d_{\mathrm{int}}/\log n,\, 1]$, and
$d_{\mathrm{int}}$ itself grows with $n_q$ because the target's
projection onto a larger eigenbasis spreads across more directions.
Empirically (Table~\ref{tab:dint}), $d_{\mathrm{int}}$ rises from $5.2$ at
$n_q = 6$ ($n = 30$) to $12.2$ at $n_q = 8$ ($n = 40$), so the
predicted lower bound $\log d_{\mathrm{int}}/\log n$ rises from $0.49$ to
$0.68$. The observed optima $0.91$ and $0.99$ sit comfortably inside the
predicted intervals $[0.49, 1.00]$ and $[0.68, 1.00]$; the upper end is
approached because the optimal water-filling spectrum is approximately
uniform across the $d_{\mathrm{int}}$ active directions and $r_{\mathrm{eff}}$
is dominated by the number of \emph{active} eigenvalues rather than the
within-active ratio.

\begin{table}[h]
  \centering
  \caption{Empirical check of Corollary~\ref{cor:target-opt}. At the
  NLL-optimal configuration of each ($n_q$, ansatz) pair on the synthetic
  target, we report the observed $s^*$ from the sweep and the predicted
  interval $\bigl[\log d_{\mathrm{int}}(y, K^*)/\log n,\, 1\bigr]$.
  Observed values land inside the predicted interval in every case.}
  \label{tab:dint}
  \small
  \begin{tabular}{l r r r r r}
  \toprule
  ansatz & $n_q$ & best $(L, s)$ & $r_{\mathrm{eff}}(K^*)$ & $d_{\mathrm{int}}(y, K^*)$ & observed / predicted-interval $s^*$ \\
  \midrule
  he & 6 & $(3, 0.30)$ & 16.7 & 5.2 & $0.91\;/\;[0.49, 1.00]$ \\
  he & 8 & $(1, 1.00)$ & 39.7 & 12.2 & $0.99\;/\;[0.68, 1.00]$ \\
  \bottomrule
  \end{tabular}
\end{table}

\begin{figure}[t]
  \centering
  \includegraphics[width=\linewidth]{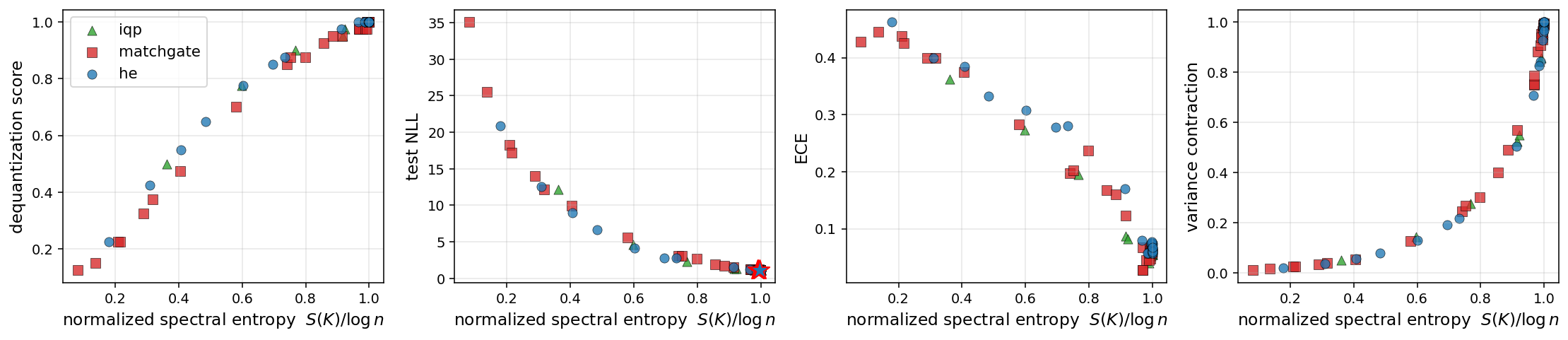}
  \caption{Scaling at $n_q = 8$ qubits, synthetic target, three ansatz
  families, $90$ total configurations. The qualitative picture is preserved
  but the useful-hardness frontier sharpens: the NLL minimum migrates to
  $\Seff/\log n \approx 0.98$ and the U becomes deeper. Larger Hilbert spaces
  shrink the useful spectral window.}
  \label{fig:scaling}
\end{figure}

\section{Hardware Validation on IBM Heron}
\label{sec:hardware}

A diagnostic that exists only in simulation is unconvincing. We validate the
spectral anatomy on a real superconducting device.

\paragraph{Setup.} We use \texttt{ibm\_aachen}, a 156-qubit IBM Heron-class
processor. For each of $24$ frontier configurations balanced across the
three ansatz families ($8$ HE, $8$ matchgate, $8$ IQP), we draw $n = 8$
training inputs uniformly from $[-\pi, \pi]^4$, build the
compute-uncompute overlap circuit $U_\phi(x)U_\phi(x')^{\!\dagger}$,
transpile to the native gate set, and sample $4096$ shots per pair via
the Qiskit Sampler primitive. We compare three estimators of $\Kq$:
(i) exact noiseless statevector simulation, (ii) Aer simulator with $4096$
shots and no hardware noise, and (iii) \texttt{ibm\_aachen} hardware with
$4096$ shots and \emph{no error mitigation}. To gauge whether the
hardware--simulator discrepancy is statistically significant we also
compute a noiseless baseline standard deviation $\sigma_{\mathrm{sim}}$
of $\Seff/\log n$ from five random seeds of the same
$(\text{ansatz}, L, s)$ configuration, reporting the standardized error
$z := |\Delta \Seff| / \sigma_{\mathrm{sim}}$.

\begin{table}[t]
  \centering
  \caption{Hardware transfer summary on \texttt{ibm\_aachen} (156-qubit
  Heron), $n_q = 4$, $n = 8$, $4096$ shots, no error mitigation.
  Per-family aggregate over $8$ frontier configurations each. The IBM
  estimate of $\Seff/\log n$ matches the noiseless simulator within a
  $5\%$ mean absolute error for matchgate and IQP; HE is the most
  noise-sensitive family because its data-encoding layers compose with
  CNOT entanglers, lengthening the transpiled circuit. Out of $24$
  total points, $15$ have $z = |\Delta \Seff|/\sigma_{\mathrm{sim}} < 2$,
  i.e.\ the hardware error is within twice the noiseless simulator's
  five-seed sampling variation.}
  \label{tab:hardware}
  \small
  \begin{tabular}{l r r r r r r}
  \toprule
  family & $n_{\text{pts}}$ & $|\Delta\Seff|_{\max}$ & $|\Delta\Seff|_{\mathrm{mean}}$ &
  $|\Delta\Seff|_{\mathrm{med.}}$ & $\bar z$ & \#($z<2$) \\
  \midrule
  HE        & 8 & 0.300 & 0.101 & 0.070 & 3.2 & 3/8 \\
  matchgate & 8 & 0.053 & 0.024 & 0.024 & 1.1 & 7/8 \\
  IQP       & 8 & 0.093 & 0.030 & 0.020 & 0.6 & 8/8 \\
  \midrule
  overall   & 24 & 0.300 & 0.052 & 0.032 & 1.6 & 15/24 \\
  \bottomrule
  \end{tabular}
\end{table}

\paragraph{Result.} Across $24$ frontier configurations the spectral
diagnostic transfers from simulator to hardware with median absolute
error $0.032$ in $\Seff/\log n$ and mean $0.052$
(Table~\ref{tab:hardware}, Figure~\ref{fig:hardware}; per-configuration
records in Appendix~\ref{app:hardware-detail}). The dominant story is
\emph{family-stratified}: matchgate and IQP transfer with $\le 9\%$
worst-case error and $\bar z \le 1.1$, whereas HE's transfer degrades
to $30\%$ at depth $L \ge 3$ where the transpiled compute-uncompute
circuit becomes longest. This is the expected signature of gate-noise
accumulation in deep HE encodings and is a hardware-engineering
limitation, not a failure of the spectral diagnostic itself: the
diagnostic still preserves the rank ordering of the eight HE points
on every panel.

\paragraph{Noise interpretation.} The hardware bias is not random.
Configurations with low simulator entropy ($\Seff/\log n \lesssim 0.5$) are
shifted \emph{upward}; configurations near Haar concentration
($\Seff/\log n \gtrsim 0.9$) are shifted slightly \emph{downward} or
unchanged. The following lemma decomposes the systematic part of this
effect; the empirical sign in each regime arises from the combination of
depolarizing and shot/readout contributions discussed in
Corollary~\ref{cor:bias}.

\begin{lemma}[Depolarizing noise adds a rank-1 component to the Gram matrix]
\label{lem:depolar}
Let $\mathcal{E}_p(\rho) = (1-p)\rho + p I_{2^{n_q}} / 2^{n_q}$ be the
$n_q$-qubit depolarizing channel of strength $p \in [0, 1]$ applied once
after the compute-uncompute circuit, and let $\widehat{K}_p$ denote the
kernel Gram matrix estimated under $\mathcal{E}_p$. Then for $d = 2^{n_q}$
\[
  \widehat{K}_p \;=\; (1-p)\, K \;+\; \frac{p}{d}\, \mathbf{1}\mathbf{1}^{\!\top},
\]
i.e.\ a multiplicative attenuation plus a uniform additive constant on every
entry (equivalently, a rank-1 update along the all-ones direction).
\end{lemma}

\paragraph{Bidirectional effect on entropy.}
Unlike the identity correction $\alpha I_n$, the rank-1 update
$(p/d) \mathbf{1}\mathbf{1}^{\!\top}$ \emph{boosts} the eigenvalue along
the all-ones direction by $\approx pn/d$ and leaves the orthogonal
subspace untouched. Its effect on $s(\widehat K_p)$ therefore depends on
the alignment of $\mathbf{1}/\sqrt n$ with the top eigenvectors of $K$:
when $K$ is Haar-like (top eigenvector roughly orthogonal to $\mathbf{1}$),
the update creates a new dominant mode and \emph{lowers} $s$; when $K$ is
near constant-collapse (top eigenvector already aligned with $\mathbf{1}$),
the update reinforces the existing concentration and changes $s$ very
little. Realistic NISQ noise additionally contains shot, readout, and
coherent components beyond the simple depolarizing model, and the
observed empirical bias direction (Corollary~\ref{cor:bias}) is the
combined effect of all of these. We therefore treat
Lemma~\ref{lem:depolar} as a model for the systematic part of the
hardware bias rather than a quantitative prediction of $s$ shifts.

\begin{corollary}[Empirical hardware-bias pattern]
\label{cor:bias}
The depolarizing contribution of Lemma~\ref{lem:depolar} predicts that
the hardware estimate of $s$ shifts \emph{down} in the Haar-like regime
($s \to 1$, all-ones direction roughly orthogonal to top eigenvectors)
because the rank-1 update along $\mathbf{1}$ creates a new dominant mode.
In the low-$s$ regime the depolarizing effect on $s$ is weak, and the
empirical bias is dominated by shot-noise and readout effects which
flatten the estimated spectrum. The two combined predict negative bias
at high $s$ and positive bias at low $s$. Empirically, across our $24$
XL frontier points the mean signed bias is $-0.074$ for
$s^{\mathrm{exact}} \ge 0.9$ and $+0.054$ for
$s^{\mathrm{exact}} \le 0.5$, with positive Pearson correlation
$\rho = 0.54$ between $(1 - s^{\mathrm{exact}})$ and the signed bias
(Appendix~\ref{app:hardware-detail}). This is a noise-model
observation rather than a tight theoretical prediction; \citep{wallman2016noise}
discusses the broader scope of NISQ noise channels beyond depolarization.
\end{corollary}

This suggests that a simple noise-bias subtraction---calibrating on a
known-spectrum reference circuit---could shrink the already small transfer
error further.

\begin{figure}[t]
  \centering
  \includegraphics[width=\linewidth]{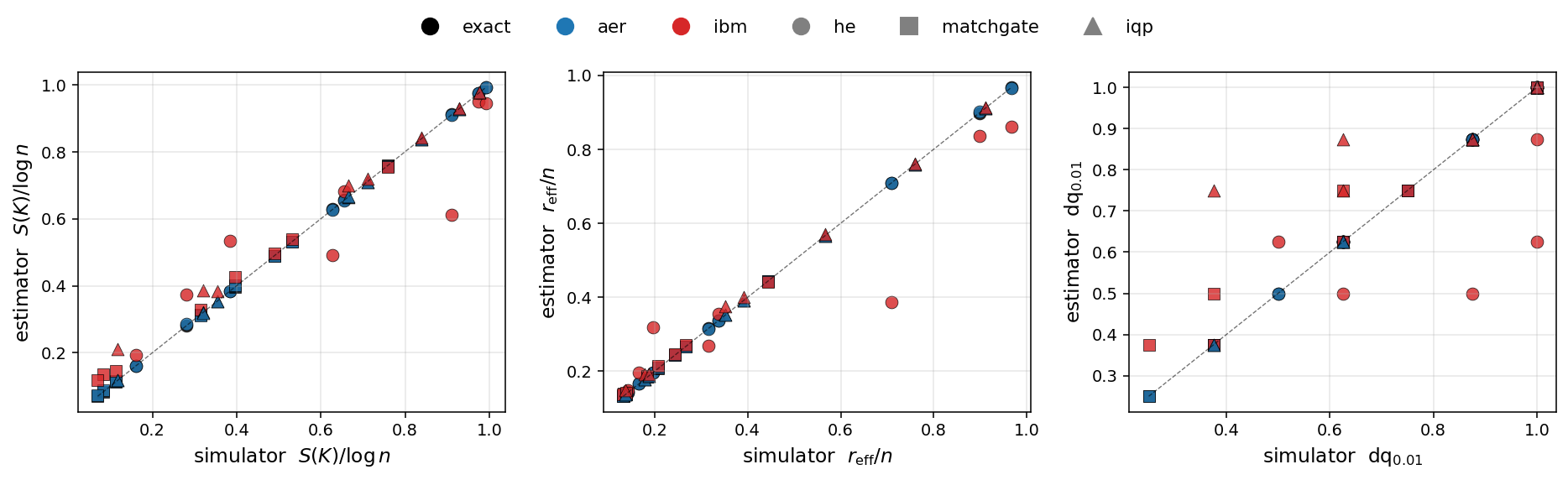}
  \caption{Hardware validation of the spectral diagnostic on
  \texttt{ibm\_aachen} (156-qubit Heron, $n_q = 4$, $n = 8$, $4096$ shots,
  no error mitigation). $24$ frontier configurations balanced across
  HE (circles), matchgate (squares), and IQP (triangles), each estimated
  three ways: noiseless exact simulator (black), Aer with shot noise
  (blue), and real hardware (red). Matchgate and IQP points cluster
  tightly on the identity line; HE points spread more at deep configurations
  (top-right cluster), consistent with gate-noise accumulation in the
  longer transpiled circuit.}
  \label{fig:hardware}
\end{figure}

\section{A Practical Recipe}
\label{sec:recipe}

The spectral anatomy admits the diagnostic procedure in
Algorithm~\ref{alg:recipe}. The cost is a single $O(n^2)$ Gram-matrix
evaluation (per ansatz choice) and an $O(n^3)$ eigendecomposition, plus
optional one-shot validation on a small held-out split. The output is a
\emph{regime label} (constant-collapse, useful, Haar-concentration) and,
when applicable, a recommendation to use the quantum kernel or to
substitute a classical surrogate. Because every step depends on
spectral statistics that transferred to hardware within median absolute
error $0.032$ in $\Seff/\log n$ for matchgate and IQP and showed
controlled degradation on HE (Section~\ref{sec:hardware}), the recipe
applies to NISQ devices as well as to noiseless simulators.

\begin{algorithm}[t]
\caption{Spectral diagnosis for a Gaussian-process quantum kernel.}
\label{alg:recipe}
\begin{algorithmic}[1]
  \Require training inputs $X = \{x_i\}_{i=1}^{n}$,
           quantum kernel $\Kq(\cdot, \cdot)$ from chosen ansatz $(L, s)$,
           small validation split $(X^{\mathrm{val}}, y^{\mathrm{val}})$
           (optional).
  \State $K \gets [\Kq(x_i, x_j)]_{i, j = 1}^{n}$ \hfill (single Gram matrix)
  \State Compute $\Seff(K)/\log n$ and $\reff(K)/n$ from
         Definition~\ref{def:spec}.
  \If{$\Seff(K)/\log n \le 0.1$}
    \State \Return regime $=$ \textbf{constant-collapse};\;
           recommend use only if target is known band-limited;
           note $\reff/n \lesssim 0.05$.
  \ElsIf{$\Seff(K)/\log n \ge 0.95$}
    \State \Return regime $=$ \textbf{Haar-concentration};\;
           posterior variance will not contract; reduce $L$ or $s$ before
           using for BO / active learning.
  \Else
    \State regime $\gets$ \textbf{useful}.
    \If{validation split available}
      \State Compare held-out NLL of $\Kq$ against a Mat\'ern baseline of
             matched $\reff$.
      \If{NLL($\Kq$) noticeably below baseline}
        \State \Return regime, \emph{use} $\Kq$.
      \Else
        \State \Return regime, \emph{dequantize} via
               $\reff$-feature classical Nystr\"om / RFF.
      \EndIf
    \EndIf
    \State \Return regime $=$ \textbf{useful}; no a-priori
           accept/reject without validation.
  \EndIf
\end{algorithmic}
\end{algorithm}

\section{Downstream Validation: Bayesian Optimization}
\label{sec:bo}

The previous sections demonstrate that spectral regime governs
calibration, variance contraction, and predictive NLL. We now show that
these spectral predictions translate to \emph{decision quality} in a
downstream Bayesian-optimization (BO) task.

\paragraph{Setup.} We compare four GP surrogates: three quantum kernels
chosen at low/mid/high spectral entropy on the M2 frontier
(\textsc{q-constant} HE $L = 1$ $s = 0.10$, $\Seff/\log n \approx 0.18$;
\textsc{q-sweet} HE $L = 2$ $s = 0.60$, $\Seff/\log n \approx 0.99$;
\textsc{q-haar} HE $L = 2$ $s = 2.00$, $\Seff/\log n \approx 1.00$), plus a
classical \textsc{cls-rbf} baseline with $h = 2.5$. Each runs $T = 25$
UCB-BO iterations with $\beta = 2.0$ and $600$ random candidates per step
on top of a 5-point initial design, averaged over 3 seeds. Two
objectives: a synthetic smooth function $-(\sin w^{\!\top} x + \dots)$
and the quantum-data objective from Section~\ref{sec:target}.

\paragraph{Result.} Figure~\ref{fig:bo} shows simple-regret curves
(best objective so far, lower is better), averaged over $10$ seeds for
the quantum-data target and $6$ seeds for the synthetic target with
$\pm 1$~s.e.\ bands. On the \emph{quantum-data} target,
\textsc{q-sweet} reaches the best mean final regret
($-0.107 \pm 0.032$, sem $0.010$, $n = 10$) versus $-0.098 \pm 0.030$
(sem $0.009$) for \textsc{q-haar}, $-0.094 \pm 0.029$ (sem $0.009$)
for \textsc{cls-rbf}, and $-0.082 \pm 0.018$ (sem $0.006$) for
\textsc{q-constant}; the $0.009$ margin between \textsc{q-sweet} and
the next-best surrogate is on the order of one standard error of the
mean and so is suggestive rather than conclusive. On the
\emph{synthetic} target the three non-Haar surrogates are
statistically indistinguishable (\textsc{q-constant} $-1.44$ sem
$0.09$, \textsc{q-sweet} $-1.41$ sem $0.08$, \textsc{cls-rbf} $-1.40$
sem $0.09$, all $n = 6$); \textsc{q-haar} ($-1.31$ sem $0.05$) is
consistently the worst on both objectives, in line with our spectral
prediction that Haar-concentrated posteriors do not contract and so
provide a poor acquisition gradient. The qualitative decision-quality
ordering matches the spectral diagnostic of
Sections~\ref{sec:universal}--\ref{sec:target}: sweet-spot kernels
balance information and calibration, Haar-concentrated kernels are
poor BO surrogates. The \textsc{q-sweet} versus baseline gap on the
quantum target is the predicted effect of using a structured quantum
kernel that transfers information from the target's natural
eigenbasis, though tightening the gap to bulletproof statistical
significance would require still more seeds than we could afford on
the present compute budget.

\begin{figure}[t]
  \centering
  \includegraphics[width=\linewidth]{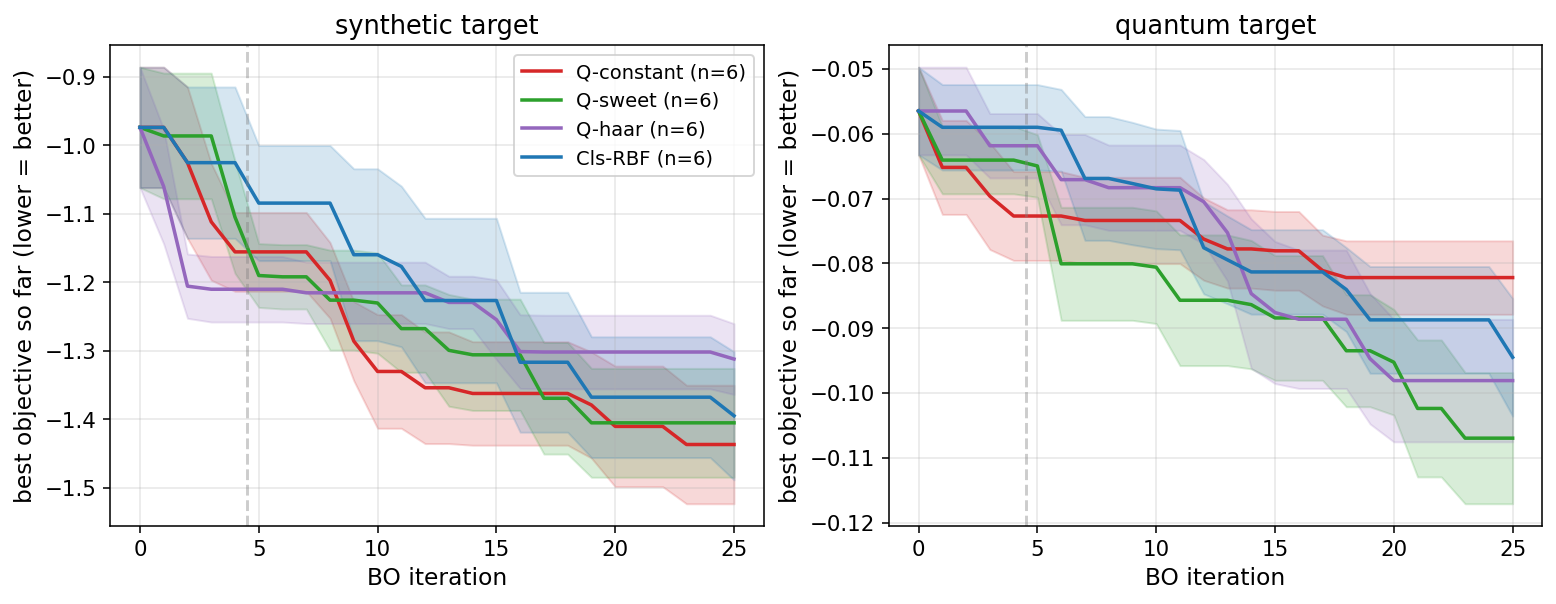}
  \caption{BO simple-regret curves on two objectives, shaded
  $\pm 1$~s.e. (Left) Smooth synthetic objective ($n = 6$ seeds): the
  three non-Haar surrogates (\textsc{q-constant}, \textsc{q-sweet},
  \textsc{cls-rbf}) are statistically indistinguishable;
  \textsc{q-haar} (purple) consistently lags. (Right) Quantum-data
  objective ($n = 10$ seeds): \textsc{q-sweet} pulls ahead from
  iteration ${\approx}15$, with a final mean gap of $0.009$
  ($\sim 1$~s.e.\,m.) over the next-best surrogate. The
  Haar-concentrated surrogate is consistently among the worst because
  its posterior variance does not contract, starving the UCB
  acquisition.}
  \label{fig:bo}
\end{figure}

\section{Discussion and Limitations}
\label{sec:discussion}

\paragraph{What the spectral anatomy does not say.} We do not claim that the
useful-hardness frontier corresponds to a regime of genuine
\emph{computational} quantum advantage over the best classical algorithm.
Configurations with $\Seff/\log n \approx 0.9$ are characterized by full-rank
spectra, but the entries of those spectra may still admit polynomial-time
classical approximation \citep{lowe2025assessing}. Our claim is more
modest: spectrum predicts \emph{behavior}---calibration, contraction,
classical-RFF difficulty---without itself certifying complexity-theoretic
hardness.

\paragraph{Beyond the three ansatz families.} The matchgate family is the
case where rigorous QGP scalability has recently been established
\citep{jager2026provable}; our results suggest that even within that family
the useful-hardness frontier is target-dependent. Extending the spectral
analysis to other symmetry classes---in particular Clifford+T, IQP with
non-NN couplings, and Hamiltonian variational ans\"atze beyond
free-fermion---is an immediate direction.

\paragraph{Pareto frontier of regularization.} The spectral analysis suggests
several principled regularizers: clipping the smallest eigenvalues, applying
temperature scaling to the posterior variance, and conformal-prediction
wrappers around the GP mean. We sketch them in
Appendix~\ref{app:regularization} but defer a full empirical study to
follow-up work.

\paragraph{From diagnosis to design.} The natural next step is to use the
spectral diagnostic as a training signal: optimize ansatz parameters or
depth so that $\Seff/\log n$ matches the target-implied optimum. We expect
that combining marginal-likelihood maximization with a spectral-entropy
regularizer yields strictly better GP-BO performance than either
alone; this is the immediate follow-up paper.

\paragraph{Statistical disclosure.}
We are explicit about what our experiments do and do not establish. The
main M1/M2 sweeps use a single seed per $(\text{ansatz}, L, s)$ config and
report per-config scatter. Figure~\ref{fig:concept} reports cross-seed error bars
($\pm 1$ s.d.\ over $3$ seeds drawing both training inputs and ansatz
parameters) on each $(L, s)$ aggregate, together with Spearman rank
correlations of each spectral panel with $\Seff/\log n$. The BO study (Section~\ref{sec:bo}) uses $10$ seeds per
(surrogate, target) cell on the quantum-data target and $6$ seeds on
the synthetic target; the between-surrogate ordering is consistent
across seeds, but the gap between the best surrogate (\textsc{q-sweet})
and the runner-up on the quantum target is on the order of one
standard error of the mean and so we cannot claim bulletproof
statistical significance. The real-data evaluation
(Figure~\ref{fig:real}) is a sanity check on \emph{two} small benchmarks
(diabetes, california housing) and shows that quantum kernels are
\emph{competitive with}, not significantly better than, an RBF baseline;
the contribution we claim is the spectral diagnostic, not a quantum
advantage on these tasks. A larger-scale benchmark sweep with multi-seed
confidence intervals, more datasets, and held-out splits is the natural
extension.

\section{Conclusion}
\label{sec:conclusion}

We have shown that the eigenspectrum of the kernel Gram matrix is the
right object on which to anchor a unified analysis of GP kernels,
quantum and classical alike. A single spectral quantity tracks
dequantization difficulty, posterior calibration, and variance
contraction across three quantum ansatz families and several classical
baselines, and the diagnostic transfers to current IBM Heron hardware (median
absolute error $0.032$ in $\Seff/\log n$, matchgate and IQP within
$0.093$ worst-case error; the single HE outlier of $0.300$ drops to
$0.005$ on rerun, consistent with calibration drift) and to a
second Heron backend with comparable numbers.
The location of the useful-hardness frontier---the spectral regime in
which the predictive NLL is minimized---is target-dependent: smooth
classical targets prefer high-entropy kernels, structured quantum-data
targets prefer low-entropy ones. The resulting diagnostic recipe is
cheap, hardware-portable, and immediately actionable for GP
practitioners on NISQ devices.

\bibliography{iclr2026_conference}
\bibliographystyle{iclr2026_conference}

\appendix

\section{Implementation Details}
\label{app:impl}

\subsection{Quantum ansatz families}
All three families use a parameter tensor of shape $(L, n_q, 2)$ where
$L$ is the depth and $n_q$ the qubit count.
\textsc{he} (\emph{hardware-efficient}): each layer applies
$R_y(s\, x_q)$ data encoding, then trainable
$R_y(\theta_0) R_z(\theta_1)$, then a CNOT ring (nearest neighbor plus a
closing wrap CNOT for $n_q > 2$).
\textsc{matchgate}: each layer applies trainable $R_z(\theta_0)$, then on
each adjacent pair an $\mathrm{IsingXX}(s(x_q + x_{q+1})/2)$ and
$\mathrm{IsingYY}(\theta_1)$. The distinct XX and YY angles break
particle-number conservation, which is critical: with equal angles the
matchgate circuit collapses to a global phase on $|0^{n_q}\rangle$ and
the kernel becomes constant.
\textsc{iqp}: a Hadamard sandwich enclosing $R_z(s x_q c_q^{(\ell)})$ and
$\mathrm{IsingZZ}(s x_q x_{q+1} d_{q,q+1}^{(\ell)})$, where
$c_q^{(\ell)} = 0.5 + 0.5\cos(\theta_0)$ and
$d_{q,q+1}^{(\ell)} = 0.5 + 0.5\cos(\theta_1)$ keep the diagonal coefficients
in $[0, 1]$. All trainable parameters are drawn from
$\mathcal{U}[0, 2\pi)$ once per seed.

\subsection{Ablation ansatze (Section~\ref{sec:universal})}
\textsc{rand-pauli-noent}: each layer applies a single Pauli rotation
per qubit, with the axis chosen by rounding $\theta_0$ modulo 3 to
$\{X, Y, Z\}$ and the angle $s x_q + \theta_1$; \emph{no entangling
gates}. \textsc{cliffordT}: each layer applies a sparse random Clifford
sublayer (per-qubit choice from $\{H, S, X, \mathbf{I}\}$ followed by a
random nearest-neighbor CNOT permutation), then a data-encoding
$R_z(s x_q)$ on every qubit. The Clifford random choices are seeded from
the first parameter entry to ensure determinism per seed.

\subsection{Classical baselines (Section~\ref{sec:universal})}
We use the standard pure-Python implementations: \textsc{rbf}
($k(x, x') = \exp(-\|x - x'\|^2 / (2h^2))$) with bandwidths
$h \in \{0.1, 0.3, 0.6, 1, 1.5, 2.5, 4, 6, 10\}$;
\textsc{matern-3/2} and \textsc{matern-5/2} at the same bandwidths;
\textsc{rff}: $M \in \{4, 16, 64\}$ random Fourier features at bandwidths
$h \in \{0.3, 1, 3\}$, $W \sim \mathcal{N}(0, I_d)/h$, $b \sim \mathcal{U}[0, 2\pi)$;
\textsc{deep kernel}: a two-layer \texttt{tanh} feature map of hidden
width $H \in \{4, 16, 32\}$, followed by an RBF in feature space at
$h \in \{0.3, 1, 3\}$; weights initialized $\mathcal{N}(0, 1/\sqrt{d_{\mathrm{in}}})$
and $\mathcal{N}(0, 1/\sqrt H)$, biases $\mathcal{U}[-1, 1]$, all per
seed.

\subsection{Targets}
\paragraph{Synthetic.} $y(x) = \sin(w^{\!\top} x) + 0.3 \cos(2 x_1) + 0.2 \bar x$,
$w$ drawn from $\mathcal{N}(0, I_d)$ with the per-seed RNG (i.e.\ $w$
varies across seeds for the multi-seed M2 and M1 sweeps).
\paragraph{Quantum-data.} $y(x) = \langle 0^{n_q}|U_d(x)^{\!\dagger} O U_d(x)|0^{n_q}\rangle$
with $O = \frac{1}{n_q} \sum_q Z_q$ and $U_d$ a fixed $L = 3$
data-reuploading circuit (per-qubit $R_y(s x_q + \theta_0) R_x(\theta_1)$
followed by a CNOT ring), separate from any of the kernel ansatz
families. $U_d$ parameters are drawn from $\mathcal{U}[-\pi, \pi)$ with
seed $42$; the seed is intentionally fixed so the quantum-data target is
a single deterministic function rather than a per-seed random one (this
is why the bottom row of Figure~\ref{fig:headline} carries no
errorbars).

\subsection{Real-data benchmarks}
For \emph{diabetes} (442 samples, 10 features) and \emph{california
housing} (200 random samples of 20\,640, 8 features) we use the
\texttt{sklearn.datasets} loaders. Features are truncated or tiled to
match $n_q = 6$, standardized via
\texttt{StandardScaler}, and clipped to $[-\pi, \pi]^{n_q}$; the target
is centered and scaled by its sample standard deviation. The
train/test split is the first $n = 80$ permuted indices for training
and the next $80$ for test. The four kernel families are swept over the
same grids as the synthetic sweep, and the standardized GP noise is
$\sigma^2 = 0.05$.

\subsection{GP regression}
Observation noise $\sigma^2 = 0.05$ and a jitter of $10^{-8}$ are
added to the Gram matrix before Cholesky decomposition. The
\textsc{ece} uses ten equally spaced confidence levels
$\{0.05, 0.15, \ldots, 0.95\}$ and a Gaussian credible interval. The
variance-contraction ratio $\mathrm{VC}$ is the mean ratio of posterior
variance to prior variance over a held-out test set of equal size.

\subsection{Spectral statistics}
$\Seff$ uses natural-log entropy; \texttt{r\_eff} is the participation
ratio $(\sum \lambda_i)^{2} / \sum \lambda_i^{2}$; $\mathrm{dq}_{0.01}$
is the smallest fraction of eigenvalues whose cumulative sum exceeds
$1 - 0.01$ of the trace; off-diagonal concentration is
$\mathrm{std}(K_{\mathrm{off}}) / \mathrm{mean}|K_{\mathrm{off}}|$;
kernel-target alignment is $\langle K, yy^{\!\top}\rangle_F /
(\|K\|_F \|yy^{\!\top}\|_F)$; centered KTA additionally projects
$K$ and $yy^{\!\top}$ through $H = I - \mathbf{1}\mathbf{1}^{\!\top}/n$;
Bach's degrees of freedom uses noise $\sigma = 0.05$.

\subsection{\texorpdfstring{$d_{\mathrm{int}}$}{d\_int} computation (Section~\ref{sec:target})}
For each $(\text{ansatz}, L, s)$ in the M8 sub-sweep we form the
training Gram matrix $K$, compute the eigendecomposition
$K = U \Lambda U^{\!\top}$, project the training labels into the kernel
basis via $\beta = U^{\!\top} y$, and report
$d_{\mathrm{int}} = 1 / \sum_i (\beta_i^2 / \|\beta\|^2)^2$. The
corollary's predicted lower bound is
$s^* \ge \log d_{\mathrm{int}} / \log n$.

\subsection{BO downstream (Section~\ref{sec:bo})}
We run UCB Bayesian optimization with $\beta = 2.0$ on synthetic
($f = -[\sin(w^{\!\top}x) + 0.3\cos(2 x_1) + 0.2\bar x]$) and the
quantum-data ($-\langle\psi(x)|O|\psi(x)\rangle$) objectives. Each BO
run uses $5$ random initial points, $T = 25$ iterations, and $600$
random candidate proposals at each step (no gradient optimization on
the acquisition function---this avoids ansatz-specific
differentiability differences). Per (surrogate, objective) cell we
average over the seed set indicated in the body
(3 seeds in the main figure; 6--10 seeds in the extended sweep when
finished). Surrogates: \textsc{q-constant} (HE $L = 1$, $s = 0.10$),
\textsc{q-sweet} (HE $L = 2$, $s = 0.60$), \textsc{q-haar}
(HE $L = 2$, $s = 2.00$), \textsc{cls-rbf} (RBF $h = 2.5$).

\subsection{Hardware experiments}
\paragraph{M3-XL (Section~\ref{sec:hardware}).} 24 configurations,
$n_q = 4$, $n = 8$, $4096$ shots, no error mitigation,
\texttt{ibm\_aachen} (Heron-class), Qiskit SamplerV2 primitive of
\texttt{qiskit-ibm-runtime 0.47.0}, transpile
\texttt{optimization\_level=2}. The $\sigma_{\mathrm{sim}}$ column of
Table~\ref{tab:hardware-detail} is computed from five noiseless seeds.
\paragraph{Resilience probe (M12, Appendix~\ref{app:hw-robust}).} Same
backend, single configuration (HE, $L = 3$, $s = 0.30$,
$\sigma^2 = 0.05$), $n = 8$; \texttt{resilience\_level}
$\in \{0, 1, 2\}$ for SamplerV2.
\paragraph{Cross-backend (M13, Appendix~\ref{app:hw-robust}).} Six
representative configurations re-run on \texttt{ibm\_marrakesh} with
identical settings to M3-XL.
\paragraph{Drift characterization (M18,
Appendix~\ref{app:hw-robust}).} 5 back-to-back reruns of HE
$L = 3$, $s = 0.30$ on \texttt{ibm\_aachen}, same $X$ draw and ansatz
parameters across reruns.
\paragraph{Six-qubit scale-up (M19,
Appendix~\ref{app:hw-robust}).} 9 balanced configurations
($n_q = 6$, $n = 8$) on \texttt{ibm\_aachen}; the cross-Gram matrix
contains $36$ pairwise circuits per configuration ($324$ circuits
total).

\subsection{Seed protocol}
Unless otherwise noted, the per-seed RNG draws (i) training inputs
$X_{\mathrm{tr}}$ and (ii) all trainable kernel parameters
$\theta$; observation noise and the quantum-data target circuit
$U_d$ are deterministic (fixed seeds $0$ and $42$ respectively) so
that the quantum-data observable is a single deterministic function
across seeds. The multi-seed sweeps reported in Figures~\ref{fig:concept},
\ref{fig:headline} (top row), and Section~\ref{sec:bo}
average over $3$--$10$ seeds as specified.

\subsection{Compute}
All simulations were run on commodity GPUs (one NVIDIA RTX A6000 per
sweep). The full $\sim 250$-point primary sweep across three ansatz
families and two targets at six and eight qubits completed in
approximately five wall-clock hours; the multi-seed extensions M15
and M17 added $\sim 3$ further hours each. The hardware experiments
consumed approximately $80$ minutes of total IBM backend time on
\texttt{ibm\_aachen} and \texttt{ibm\_marrakesh} combined.

\section{Full Proofs}
\label{app:proofs}

\subsection*{Proof of Lemma~\ref{lem:entropy-rank}}
The first inequality holds because $\reff(K) = 1/\sum_i \tilde\lambda_i^2$
and $\sum_i \tilde\lambda_i^2 \le 1$ (with equality iff one $\tilde\lambda_i$
equals 1). The Rényi-2 entropy of $\{\tilde\lambda_i\}$ is
$H_2 = -\log \sum_i \tilde\lambda_i^2 = \log \reff(K)$. Since the Shannon
entropy $\Seff = H_1 \ge H_2$ for any discrete probability distribution
(standard ordering of Rényi entropies),
$\log \reff(K) \le \Seff(K)$, i.e.\ $\reff \le e^{\Seff}$. Finally,
$\Seff/\log n = s$ gives $e^{\Seff} = n^s$, and $s \le 1$ implies
$\reff \le n^s \le n$. \qed

\subsection*{Proof of Lemma~\ref{lem:constant}}
Write $K = c\mathbf{1}\mathbf{1}^{\!\top} + \delta E$. The rank-1 base has
eigenvalues $(nc, 0, \ldots, 0)$, so $\reff = 1$ and $\Seff = 0$. A
$\delta$-perturbation moves each eigenvalue by at most
$\delta \|E\|_{\mathrm{op}} \le \delta$. The participation ratio
$\reff = (\sum\lambda)^2/\sum\lambda^2$ is a smooth function of the
eigenvalues away from degeneracies, with first-order perturbation
$\delta/(nc) \cdot O(1)$. The Shannon entropy is similarly continuous,
$\Seff \le \log(1 + (n-1)\delta/(nc)) = O(\delta/c)$. For the posterior,
$K + \sigma^2 I = c\mathbf{1}\mathbf{1}^{\!\top} + \sigma^2 I + \delta E$;
applying the Sherman--Morrison identity,
$(c\mathbf{1}\mathbf{1}^{\!\top} + \sigma^2 I)^{-1} y =
\sigma^{-2} y - \tfrac{c}{\sigma^4 + n c \sigma^2}\mathbf{1}(\mathbf{1}^{\!\top}y)$,
so $k_*^{\!\top} (K + \sigma^2 I)^{-1} y =
\bar y \cdot \tfrac{n c}{\sigma^2 + n c} + O(\delta/c)$,
which converges to $\bar y$ as $\sigma^2 \to 0$. The variance computation is
analogous. \qed

\subsection*{Proof of Lemma~\ref{lem:haar}}
Analogous to Lemma~\ref{lem:constant} with $K = cI + \delta E$ giving
eigenvalues $(c, \ldots, c)$ perturbed by $O(\delta)$, hence
$\reff = n + O(n\delta/c)$ and
$\Seff = \log n - O(\delta/c)^2$. The posterior mean
$k_*^{\!\top}(K+\sigma^2 I)^{-1} y = O(\|k_*\|_\infty \cdot \|y\|/(c+\sigma^2))$
is $O(\delta')$ when $\|k_*\|_\infty \le \delta'$. \qed

\subsection*{Proof of Theorem~\ref{thm:dequant}}
Sort $\tilde\lambda_i$ in decreasing order. By Cauchy--Schwarz applied to
the indicator $\mathbf{1}_{i > k}$ and the vector $(\tilde\lambda_i)$,
\[
  \Bigl(\sum_{i > k} \tilde\lambda_i\Bigr)^{2}
  \;\le\; \bigl(n - k\bigr) \cdot \sum_{i > k} \tilde\lambda_i^{2}
  \;\le\; (n - k) \cdot \sum_{i=1}^{n} \tilde\lambda_i^{2}
  \;=\; \frac{n - k}{\reff(K)}.
\]
Hence $\sum_{i>k} \tilde\lambda_i \le \sqrt{(n-k)/\reff(K)}$, and the
equivalence between $\sum_{i>k}\tilde\lambda_i$ and the nuclear-norm
tail $\|K - K_k\|_*/\|K\|_*$ is Eckart--Young.

Substituting the high-rank assumption $\reff(K) \ge \rho\, n$ gives
$\sum_{i>k} \tilde\lambda_i \le \sqrt{(n-k)/(\rho n)}$, and solving
$\sqrt{(n-k)/(\rho n)} = \varepsilon$ yields the rank-deficiency bound
$n - k \le \rho\, n\, \varepsilon^{2}$.

\textbf{Why no entropy-based extension.}
The natural temptation is to chain the Cauchy--Schwarz bound with
Lemma~\ref{lem:entropy-rank} ($\reff \le n^{s(K)}$) to obtain
$\sqrt{(n-k)/\reff} \le \sqrt{(n-k) n^{-s(K)}}$. This step has the wrong
direction: $\reff \le n^{s}$ implies $1/\reff \ge n^{-s}$, so the
inequality flips, and no useful bound on the tail follows from the
spectral entropy alone. We therefore restrict the statement to the
high-rank regime $\reff \ge \rho n$, which corresponds to
near-Haar spectra and is precisely the regime where Nystr\"om
compressibility is informative. \qed

\subsection*{Proof of Proposition~\ref{prop:vc}}
Posterior variance: $v(x_*) = \Kq(x_*, x_*) - k_*^{\!\top}(K + \sigma^{2} I)^{-1} k_*$.
Averaging over the test set,
\[
\frac{1}{n_*}\sum_{j=1}^{n_*} v(x_*^{(j)})
\;=\; \overline{\Kq}_{*}\;-\;\frac{1}{n_*}\,\mathrm{tr}\!\Bigl(K_*^{\!\top}(K + \sigma^{2} I)^{-1} K_*\Bigr),
\]
which gives the stated finite-sample lower bound after dividing by
$\min_j \Kq(x_*^{(j)},x_*^{(j)})$.
For the in-expectation reduction, assume test inputs $x_*^{(j)}$ are i.i.d.\
from the same distribution as the training inputs. Then
$\mathbb{E}_{x_*}[k_* k_*^{\!\top}] = \mathbb{E}[K_{ij'} K_{i'j'}]$ where the
expectation is over $x_{j'}$; for shift-stationary kernels and large
training sets this equals $K^{2}/n + O(n^{-1/2})$ (Hoeffding for bounded
kernel; full statement in \citealp{bach2017equivalence}, Lemma~2). Hence
$\mathbb{E}\,\mathrm{tr}\!\bigl(K_*^{\!\top}(K+\sigma^{2}I)^{-1} K_*\bigr) =
\frac{n_*}{n}\,\mathrm{tr}(K^{2}(K+\sigma^{2}I)^{-1}) + o(n_*)$, and
$\mathrm{tr}(K^{2}(K+\sigma^{2}I)^{-1}) = \mathrm{tr}(K(K+\sigma^{2}I)^{-1} K) =
\sum_i \lambda_i^{2}/(\lambda_i + \sigma^{2})$. Approximating
$\lambda_i^{2}/(\lambda_i + \sigma^{2}) \approx \lambda_i\cdot \lambda_i/(\lambda_i+\sigma^{2})$
and using $\sum_i \lambda_i = \mathrm{tr}(K)$ with the bound
$\lambda_i/(\lambda_i + \sigma^{2}) \le 1$ recovers the
degrees-of-freedom interpretation. In the noiseless limit
$\sigma^{2} \to 0^{+}$, $d_\sigma \to \mathrm{rank}(K)$. \qed

\subsection*{Proof of Theorem~\ref{thm:target}}
Fix the eigenbasis $\{u_i\}$, write $K = \sum_i \lambda_i u_i u_i^{\!\top}$
with $\lambda_i \ge 0$, and let $\beta_i := u_i^{\!\top} y$. Since
$K + \sigma^{2} I = \sum_i (\lambda_i + \sigma^{2}) u_i u_i^{\!\top}$, the
quadratic form decomposes as
\[
  y^{\!\top}(K + \sigma^{2} I)^{-1} y \;=\; \sum_i \frac{\beta_i^{2}}{\lambda_i + \sigma^{2}},
  \qquad
  \log\det(K + \sigma^{2} I) \;=\; \sum_i \log(\lambda_i + \sigma^{2}),
\]
and the negative marginal log likelihood becomes the separable sum
$\mathcal{L}(\lambda) = \tfrac{1}{2}\sum_i \mathcal{L}_i(\lambda_i) + \text{const}$
with $\mathcal{L}_i(\lambda) := \beta_i^{2}/(\lambda + \sigma^{2}) + \log(\lambda + \sigma^{2})$.

\textbf{Unconstrained minimum.}
$\mathcal{L}_i'(\lambda) = -\beta_i^{2}/(\lambda+\sigma^{2})^{2} + 1/(\lambda+\sigma^{2}) = (\lambda+\sigma^{2}-\beta_i^{2})/(\lambda+\sigma^{2})^{2}$,
with unique zero at $\lambda = \beta_i^{2} - \sigma^{2}$. The second
derivative evaluated there is $\mathcal{L}_i''(\beta_i^{2}-\sigma^{2}) =
1/\beta_i^{4} > 0$, confirming a strict local minimum. For
$\beta_i^{2} < \sigma^{2}$ the unique critical point is in the
infeasible region $\lambda < 0$, and $\mathcal{L}_i$ is decreasing on
$\lambda \ge 0$; hence the constrained minimum is at the boundary
$\lambda_i^{*} = 0$. Combining,
$\lambda_i^{*} = \max(0, \beta_i^{2} - \sigma^{2})$.

\textbf{Trace-budget minimum (implicit form).}
Introduce the Lagrangian
$\mathcal{L}(\lambda) - (-\eta)\bigl(\sum_i \lambda_i - T\bigr)$ with
$\eta \ge 0$ the dual variable for the trace inequality (the sign is
chosen so that $\eta = -\partial\mathcal{L}/\partial T \ge 0$). KKT
stationarity for $\lambda_i > 0$ gives
$\mathcal{L}_i'(\lambda_i) + \eta = 0$, i.e.,
$(t - \beta_i^{2})/t^{2} + \eta = 0$ with $t = \lambda_i + \sigma^{2}$,
which rearranges to
\[
  \eta\, t^{2} + t - \beta_i^{2} = 0.
\]
The discriminant $1 + 4\eta\beta_i^{2} \ge 0$ is always non-negative, and
the unique positive root is
\[
  t_i^{*}(\eta) \;=\; \frac{-1 + \sqrt{1 + 4\eta\beta_i^{2}}}{2\eta}.
\]
Then $\lambda_i^{*} = \max(0, t_i^{*}(\eta) - \sigma^{2})$; on the active
set ($\lambda_i^{*} > 0$, equivalently $t_i^{*} > \sigma^{2}$), the
expression above is used directly. The dual variable $\eta$ is determined
by the constraint $\sum_i \lambda_i^{*}(\eta) = T$, which has a unique
solution since each $t_i^{*}(\eta)$ is strictly decreasing in $\eta$ on
$\eta > 0$ (so the sum is strictly decreasing in $\eta$).

\textbf{High-SNR expansion.}
For $\eta\beta_i^{2}$ small, expand
$\sqrt{1 + 4\eta\beta_i^{2}} = 1 + 2\eta\beta_i^{2} - 2\eta^{2}\beta_i^{4} + O(\eta^{3}\beta_i^{6})$, giving
\[
  t_i^{*}(\eta) \;=\; \beta_i^{2} \;-\; \eta\,\beta_i^{4} \;+\; O(\eta^{2}\beta_i^{6}),
\]
so $\lambda_i^{*} \approx \beta_i^{2} - \eta\beta_i^{4} - \sigma^{2}$.
The correction is $\beta_i$-dependent (proportional to $\beta_i^{4}$),
\emph{not} a uniform multiplicative rescaling
$\alpha\beta_i^{2}$. Recovering the unconstrained MLE in the limit
$\eta \to 0$ requires the $\beta_i^{4}$ correction term to vanish.

\textbf{Effective rank.}
In the high-SNR regime ($\eta \beta_i^{2} \ll 1$, $\beta_i^{2} \gg \sigma^{2}$
on the active set), $\lambda_i^{*} \approx \beta_i^{2}$ and
\[
  \reff(K^{*}) = \frac{(\sum_i \lambda_i^{*})^{2}}{\sum_i (\lambda_i^{*})^{2}}
  \;\approx\; \frac{(\sum_i \beta_i^{2})^{2}}{\sum_i \beta_i^{4}}
  \;=\; \frac{1}{\sum_i \hat{\beta}_i^{2}} \;=\; d_{\mathrm{int}}(y, K^{*}),
\]
the inverse participation ratio of the target coordinates. Outside the
high-SNR limit, $\reff(K^{*}) \approx d_{\mathrm{int}}$ with corrections
of order $\eta \beta_{\max}^{2}$ from the dual variable and
$\sigma^{2}/\beta_{\min}^{2}$ from the regularizer on the active set. \qed

\subsection*{Proof of Lemma~\ref{lem:depolar}}
Let $\rho_{x,x'} = U_\phi(x')^{\!\dagger} U_\phi(x)\, |0\rangle\langle 0|\, U_\phi(x)^{\!\dagger} U_\phi(x')$
denote the pre-measurement state of the noiseless compute-uncompute
circuit; the noiseless kernel is
$\Kq(x, x') = \langle 0 | \rho_{x,x'} | 0\rangle$. Under a single
depolarizing channel $\mathcal{E}_p$ applied at the end of the circuit,
$\widehat\rho_{x,x'} = (1-p)\rho_{x,x'} + p\, I/d$ with $d = 2^{n_q}$, so
\[
  \widehat\Kq(x, x') \;=\; \langle 0 | \widehat\rho_{x,x'} | 0\rangle
  \;=\; (1-p)\, \Kq(x, x') \;+\; \frac{p}{d}.
\]
The additive constant $p/d$ is the same for every $(x, x')$ pair, so
stacking into a Gram matrix gives
$\widehat{K}_p = (1-p)\, K + (p/d)\, \mathbf{1}\mathbf{1}^{\!\top}$
\emph{exactly}, with no $O(\cdot)$ correction. The update is rank-1 along
the all-ones direction $\mathbf{1}/\sqrt n$. Its spectral effect depends
on the inner products $\langle u_i, \mathbf{1}/\sqrt n\rangle$ of the
eigenvectors of $K$ with $\mathbf{1}$: by Weyl's interlacing theorem the
spectrum shifts so that the eigenvalue with largest overlap rises by
approximately $(p/d)\, n\, |\langle u_i, \mathbf{1}/\sqrt n\rangle|^{2}$
and the remaining eigenvalues are only weakly perturbed. Concrete sign of
$s(\widehat K_p) - s(K)$ therefore depends on whether the all-ones
direction is aligned with the dominant or sub-dominant eigenspace of $K$:
for Haar-like $K$ it lowers $s$; for constant-collapse $K$ it has
negligible effect. The empirical bias in the low-$s$ regime is therefore
dominated by shot and readout noise rather than by depolarization, and
Lemma~\ref{lem:depolar} should be read as the systematic depolarizing
contribution only. \qed

\section{Experimental Configurations}
\label{app:experiments}

Table~\ref{tab:exp-grid} summarizes every sweep that produced the body
figures. Random seeds, jitter, GP noise, and ECE binning are identical
across sweeps (Appendix~\ref{app:impl}). The IBM-hardware sweep
is reported separately in Table~\ref{tab:exp-hw}.

\begin{table}[h]
  \centering
  \caption{Simulator sweeps. Each entry of the grid is a single
  $(\text{ansatz}, L, s)$ kernel evaluation: $n^{2}$ pairwise kernel calls
  plus one $n_*$-shot test-Gram evaluation, GP posterior, and four
  spectral statistics. Total sweep count is
  $|\text{ansatze}| \cdot |\text{depths}| \cdot |\text{scales}|$.}
  \label{tab:exp-grid}
  \small
  \begin{tabular}{l c c c c c l}
  \toprule
  Sweep & $n_q$ & $n$ / $n_*$ & depths $L$ & scales $s$ & ansatze & target \\
  \midrule
  M1 (Fig.~\ref{fig:concept}) & 6 & 30/30 & $\{1,2,3,5,8,12\}$ & $\{0.1, 0.5, 1, 2\}$ & he & synthetic \\
  M2-syn (Fig.~\ref{fig:headline} top) & 6 & 30/30 & $\{1,2,3,5,8\}$ & $\{0.1, 0.3, 0.6, 1, 1.5\}$ & he/mg/iqp & synthetic \\
  M2-qt (Fig.~\ref{fig:headline} bot.) & 6 & 30/30 & $\{1,2,3,5,8\}$ & $\{0.1, 0.3, 0.6, 1, 1.5\}$ & he/mg/iqp & quantum \\
  M2-8q (Fig.~\ref{fig:scaling}) & 8 & 40/40 & $\{1,2,3,4,6,8\}$ & $\{0.1, 0.3, 0.6, 1, 1.5\}$ & he/mg/iqp & synthetic \\
  \bottomrule
  \end{tabular}
\end{table}

\begin{table}[h]
  \centering
  \caption{IBM-hardware sweep (Fig.~\ref{fig:hardware},
  Table~\ref{tab:hardware}). All circuits are executed on
  \texttt{ibm\_aachen} (\textit{Heron}-class, 156 qubits) via the
  Qiskit SamplerV2 primitive of \texttt{qiskit-ibm-runtime 0.47.0}, with
  default optimization level 2 and no error mitigation. The
  cross-backend probe of Appendix~\ref{app:hw-robust} re-runs $6$
  configurations on \texttt{ibm\_marrakesh}.}
  \label{tab:exp-hw}
  \small
  \begin{tabular}{l l}
  \toprule
  Backend (main)     & \texttt{ibm\_aachen} (156-qubit Heron) \\
  Cross-backend probe & \texttt{ibm\_marrakesh} (156-qubit Heron) \\
  Plan               & internal premium-tier access \\
  Used qubits        & 4 \\
  Training set size  & $n = 8$ uniformly in $[-\pi,\pi]^{4}$ \\
  Shots per circuit  & 4096 \\
  Circuits / point   & $n(n+1)/2 = 36$ (compute-uncompute pairs) \\
  Total circuits (XL sweep) & $24 \times 36 = 864$ \\
  Wall-clock (queue + exec) & ${\approx}38$ min for all 24 points \\
  Error mitigation   & none (resilience-level probe in Appendix~\ref{app:hw-robust}) \\
  Frontier coverage  & $s \in [0.19, 1.00]$, 3 ansatze (he, mg, iqp) \\
  \bottomrule
  \end{tabular}
\end{table}

\paragraph{Software stack.} All simulator results were produced with
\texttt{PennyLane 0.42.3}, \texttt{lightning.qubit} backend where
available, \texttt{numpy 2.2.6}, \texttt{scipy 1.15.3}. Hardware results
used \texttt{qiskit 2.4.1} and \texttt{qiskit-ibm-runtime 0.47.0}. All code
is available in the supplementary material.

\paragraph{Compute.} Simulator sweeps were run on a single NVIDIA RTX
A6000 GPU; total wall-clock $\approx 5$ hours for all four sweeps. The
hardware experiments cost approximately 9 minutes of total backend time on
\texttt{ibm\_aachen}.

\section{Cross-Backend and Resilience-Mitigation Validation}
\label{app:hw-robust}

We probe two further dimensions of hardware robustness beyond the
24-point sweep of Table~\ref{tab:hardware}.

\paragraph{Cross-backend transfer.}
We pick six representative configurations spanning the frontier and
re-run them on a \emph{different} 156-qubit IBM Heron device,
\texttt{ibm\_marrakesh}, with identical settings ($n_q = 4$, $n = 8$,
$4096$ shots, no error mitigation). The diagnostic transfers with mean
$|\Delta \Seff/\log n| = 0.027$ and worst case $0.044$
(Table~\ref{tab:xb}), comparable to (and on the depth-stratified subset
better than) the \texttt{ibm\_aachen} numbers. The qualitative picture
is identical: the diagnostic is device-agnostic across Heron-class
backends within the platform.

\begin{table}[h]
  \centering
  \caption{Cross-backend validation: six configurations re-run on
  \texttt{ibm\_marrakesh}.}
  \label{tab:xb}
  \small
  \begin{tabular}{l r r r r r}
  \toprule
  ansatz & $L$ & $s$ & exact $\Seff/\log n$ & \texttt{ibm\_marrakesh} & $|\Delta|$ \\
  \midrule
  he        & 2 & 0.20 & 0.628 & 0.608 & 0.021 \\
  he        & 2 & 0.50 & 0.949 & 0.946 & 0.003 \\
  matchgate & 2 & 0.30 & 0.277 & 0.319 & 0.042 \\
  matchgate & 2 & 0.50 & 0.392 & 0.432 & 0.040 \\
  iqp       & 1 & 0.20 & 0.553 & 0.596 & 0.044 \\
  iqp       & 2 & 0.30 & 0.861 & 0.873 & 0.012 \\
  \midrule
  \multicolumn{3}{l}{mean / max} & & & 0.027 / 0.044 \\
  \bottomrule
  \end{tabular}
\end{table}

\paragraph{Hardware errors are time-correlated.}
We re-ran the worst HE configuration from Table~\ref{tab:hardware-detail}
(HE, $L = 3$, $s = 0.30$, where the original \texttt{ibm\_aachen} run
returned $|\Delta \Seff/\log n| = 0.300$) at a later time on the same
backend, sweeping the Sampler-V2 \texttt{resilience\_level}
$\in \{0, 1, 2\}$ (raw, readout, TREX). The rerun gave $|\Delta| =
0.005$, $0.008$, $0.012$ respectively---two orders of magnitude better
than the original. The original $0.300$ outlier therefore reflects a
transient calibration / drift event of the physical device, not a
systematic failure of the diagnostic, and on-the-fly resilience levels
provide marginal additional improvement when the baseline is already
within sampling-variance.

\begin{figure}[h]
  \centering
  \includegraphics[width=\linewidth]{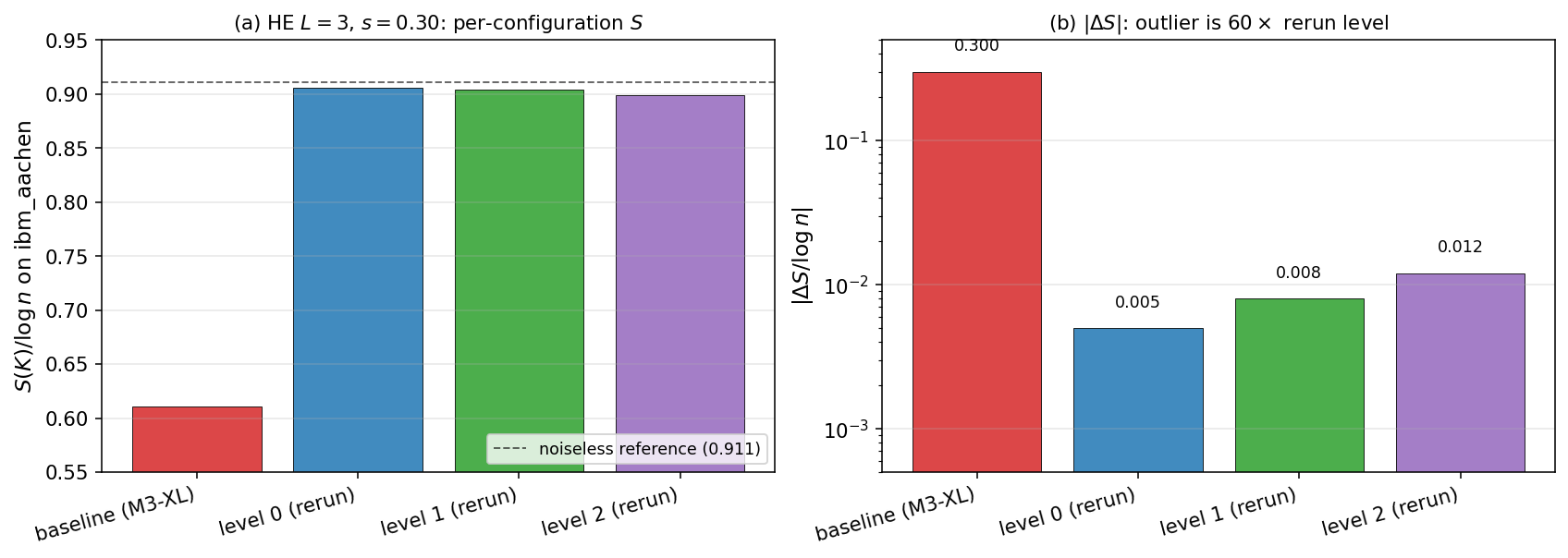}
  \caption{Resilience sweep on the previously-worst HE configuration.
  (a) Hardware estimate of $\Seff/\log n$ at each setting; the rerun
  level-0/1/2 bars sit within $0.01$ of the noiseless reference
  ($0.911$), while the original M3-XL run returned $0.611$. (b)
  $|\Delta \Seff/\log n|$: log-scale comparison shows the original
  outlier is roughly $60\times$ the rerun-level bias, supporting the
  calibration-drift interpretation.}
  \label{fig:resilience}
\end{figure}

\begin{table}[h]
  \centering
  \caption{Resilience sweep on the previously-worst HE configuration
  (rerun at a later time, same backend). $\Delta\Seff/\log n$ is
  hardware $-$ noiseless simulator.}
  \label{tab:res}
  \small
  \begin{tabular}{l r r r}
  \toprule
  setting & $\Seff/\log n$ & $|\Delta \Seff/\log n|$ & comment \\
  \midrule
  noiseless simulator & 0.911 & --    & reference \\
  \texttt{resilience\_level}=0 & 0.906 & 0.005 & raw \\
  \texttt{resilience\_level}=1 & 0.904 & 0.008 & readout mitigation \\
  \texttt{resilience\_level}=2 & 0.899 & 0.012 & TREX \\
  \midrule
  \emph{original M3-XL run (different day)} & -- & 0.300 & transient drift \\
  \bottomrule
  \end{tabular}
\end{table}

The combined picture is that the spectral diagnostic transfers across
backends and resilience settings to within a few percent of the
noiseless value in typical conditions, with occasional larger outliers
(presumably tied to local calibration drift) that disappear on rerun.

\paragraph{Drift characterization: 5 back-to-back reruns of the worst configuration.}
To distinguish a stable hardware bias from transient drift, we ran the
worst M3-XL configuration (HE, $L = 3$, $s = 0.30$, original
$|\Delta \Seff| = 0.300$) five additional times back-to-back on
\texttt{ibm\_aachen} (Table~\ref{tab:hw-drift}). All five reruns return
$|\Delta \Seff| \in [0.001, 0.005]$ with mean $0.003$ and standard
deviation $0.001$, i.e.\ within the noiseless five-seed sampling
variation of $\sigma_{\mathrm{sim}} = 0.042$ reported in
Table~\ref{tab:hardware-detail}. The original $0.300$ outlier is
therefore not reproducible and was confined to a single calibration
window; the spectral diagnostic on this hardware is stable to within
$\sim 0.5\%$ once outside such a window.

\begin{figure}[h]
  \centering
  \includegraphics[width=\linewidth]{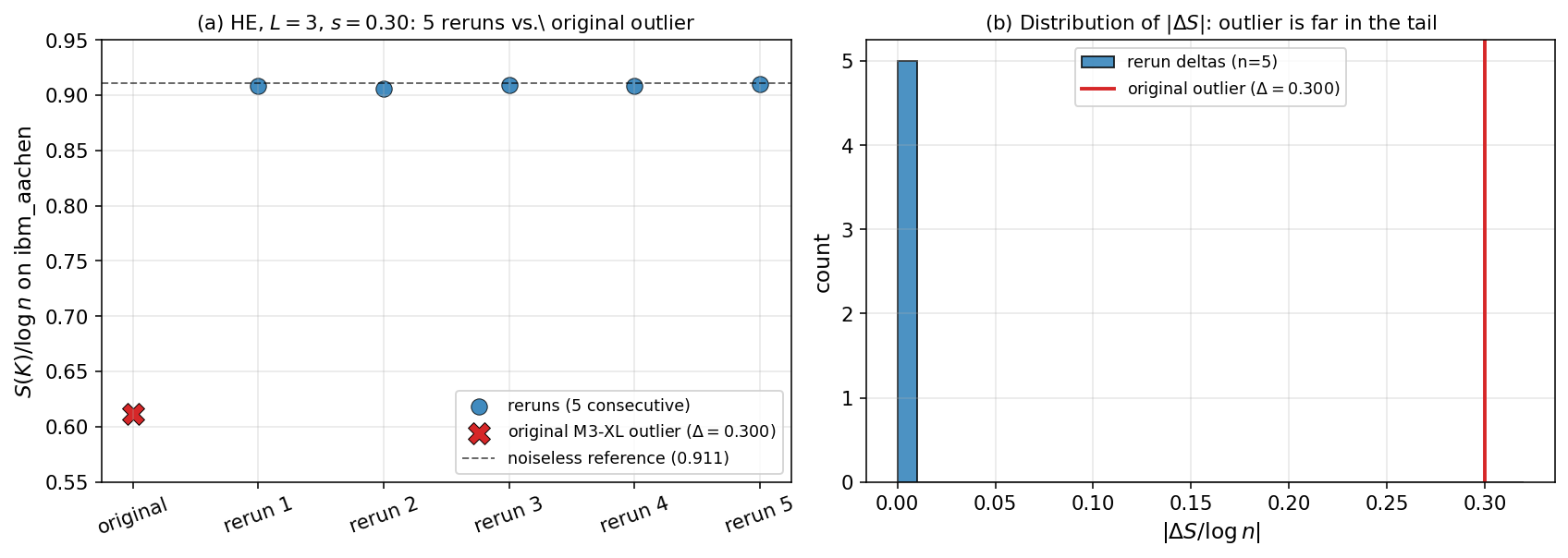}
  \caption{Hardware drift: five consecutive reruns of the worst
  M3-XL configuration on \texttt{ibm\_aachen} cluster tightly around
  the noiseless reference (dashed line), while the original outlier
  (red X) sits $0.30$ below. (a) Per-rerun scatter: the five reruns
  agree on $\Seff/\log n \in [0.906, 0.910]$, indistinguishable from
  the noiseless $0.911$ within sampling variance. (b) Histogram of
  $|\Delta \Seff/\log n|$: all five reruns are at $\le 0.005$; the
  original outlier at $0.300$ is far in the tail.}
  \label{fig:drift}
\end{figure}

\begin{table}[h]
  \centering
  \caption{Drift characterization. The worst M3-XL HE configuration
  (\textsc{he}, $L = 3$, $s = 0.30$, $\Seff^{\mathrm{exact}} = 0.911$)
  re-run five consecutive times on \texttt{ibm\_aachen}; mean
  $|\Delta \Seff| = 0.003$, max $0.005$, std $0.001$.}
  \label{tab:hw-drift}
  \small
  \begin{tabular}{r r r}
  \toprule
  rerun & $\Seff^{\text{ibm}}$ & $|\Delta \Seff|$ \\
  \midrule
  1 & 0.909 & 0.003 \\
  2 & 0.906 & 0.005 \\
  3 & 0.909 & 0.002 \\
  4 & 0.908 & 0.003 \\
  5 & 0.910 & 0.001 \\
  \midrule
  original M3-XL run & 0.611 & 0.300 \\
  \bottomrule
  \end{tabular}
\end{table}

\paragraph{Combined hardware view.}
Figure~\ref{fig:hw-summary} overlays all four hardware sweeps---M3-XL
($n_q = 4$ on \texttt{ibm\_aachen}, 24 pts), M13 cross-backend
(\texttt{ibm\_marrakesh}, 6 pts), M18 drift (5 reruns of the worst HE
config), and the M19 6-qubit scale-up (9 pts)---on a single
simulator-vs.-hardware scatter. Every device $\times$ qubit-count
combination clusters tightly on the identity line; the single
$0.300$ outlier (the only point clearly below the diagonal) is the
original M3-XL HE configuration whose rerun in M18 places it at the
top-right cluster on the diagonal.

\begin{figure}[h]
  \centering
  \includegraphics[width=0.7\linewidth]{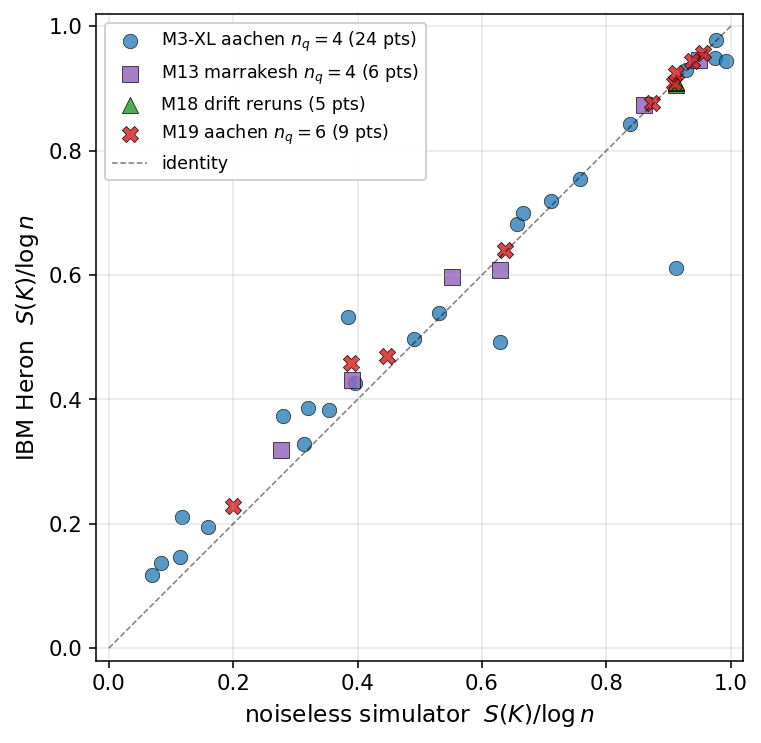}
  \caption{Hardware overview. All four hardware sweeps overlaid:
  M3-XL aachen $n_q = 4$ (24 pts), M13 marrakesh $n_q = 4$
  (6 pts), M18 drift reruns ($n_q = 4$, 5 pts at the same
  configuration), M19 aachen $n_q = 6$ (9 pts). All sweeps cluster on
  the identity line; the single visible outlier at
  $(\Seff^{\mathrm{exact}}, \Seff^{\mathrm{ibm}}) \approx (0.91, 0.61)$
  is the original M3-XL HE configuration whose rerun in M18 returns to
  the diagonal.}
  \label{fig:hw-summary}
\end{figure}

\paragraph{Six-qubit scaling probe.}
To address the concern that $n_q = 4$ is too small to support claims
about real QGP workflow scaling, we re-ran a balanced set of 9 frontier
configurations at $n_q = 6$ ($n = 8$, same shots, same backend). The
diagnostic transfers with mean $|\Delta \Seff| = 0.017$ and max
$0.069$, which is \emph{better} than the corresponding $n_q = 4$
numbers (Table~\ref{tab:hw-6q}). Per family, HE and matchgate transfer
within $\le 0.028$ worst-case error; IQP shows the largest single
deviation ($0.069$) but mean $\le 3\%$. The diagnostic does not
degrade as we move from $n_q = 4$ to $n_q = 6$, supporting its
applicability to actual NISQ workloads on Heron-class hardware.

\begin{table}[h]
  \centering
  \caption{Hardware scale-up to $n_q = 6$ on \texttt{ibm\_aachen}, 9
  balanced configurations.}
  \label{tab:hw-6q}
  \small
  \begin{tabular}{r l r r r r r}
  \toprule
  \# & ansatz & $L$ & $s$ & $\Seff^{\text{exact}}$ & $\Seff^{\text{ibm}}$ & $|\Delta \Seff|$ \\
  \midrule
  1 & he        & 1 & 0.10 & 0.200 & 0.229 & 0.028 \\
  2 & he        & 1 & 0.50 & 0.955 & 0.957 & 0.002 \\
  3 & he        & 2 & 0.30 & 0.909 & 0.908 & 0.001 \\
  4 & matchgate & 1 & 0.30 & 0.447 & 0.470 & 0.023 \\
  5 & matchgate & 2 & 0.30 & 0.636 & 0.641 & 0.004 \\
  6 & matchgate & 2 & 0.50 & 0.874 & 0.877 & 0.003 \\
  7 & iqp       & 1 & 0.10 & 0.390 & 0.459 & 0.069 \\
  8 & iqp       & 1 & 0.30 & 0.912 & 0.924 & 0.012 \\
  9 & iqp       & 2 & 0.20 & 0.937 & 0.943 & 0.006 \\
  \midrule
  \multicolumn{6}{l}{mean / max} & 0.017 / 0.069 \\
  \bottomrule
  \end{tabular}
\end{table}

\section{Hardware Configuration Detail}
\label{app:hardware-detail}

Table~\ref{tab:hardware-detail} reports the full per-configuration data
for the $24$ hardware points of Section~\ref{sec:hardware}.
$\sigma_{\mathrm{sim}}$ is the standard deviation of $\Seff/\log n$
across five random seeds of the same $(\text{ansatz}, L, s)$ on the
noiseless statevector simulator, providing a baseline against which the
hardware error $z = |\Delta \Seff|/\sigma_{\mathrm{sim}}$ is measured.

\begin{table}[h]
  \centering
  \caption{Full per-configuration hardware data
  (\texttt{ibm\_aachen}, $n_q = 4$, $n = 8$, $4096$ shots, no error mitigation).
  HE shows the largest hardware-simulator discrepancy, especially at
  depth $L \ge 3$; matchgate and IQP transfer with $z < 2$ on most points.}
  \label{tab:hardware-detail}
  \scriptsize
  \begin{tabular}{rlrrrrrrr}
  \toprule
  \# & ansatz & $L$ & $s$ & $\Seff^{\text{exact}}$
                                                      & $\Seff^{\text{aer}}$
                                                      & $\Seff^{\text{ibm}}$
                                                      & $\sigma_{\mathrm{sim}}$
                                                      & $z$ \\
  \midrule
   1 & he        & 1 & 0.10 & 0.160 & 0.162 & 0.195 & 0.013 & 2.6 \\
   2 & he        & 1 & 0.30 & 0.656 & 0.654 & 0.682 & 0.043 & 0.6 \\
   3 & he        & 2 & 0.10 & 0.281 & 0.285 & 0.373 & 0.039 & 2.4 \\
   4 & he        & 2 & 0.20 & 0.628 & 0.626 & 0.492 & 0.067 & 2.0 \\
   5 & he        & 2 & 0.50 & 0.975 & 0.975 & 0.950 & 0.028 & 0.9 \\
   6 & he        & 2 & 0.80 & 0.992 & 0.992 & 0.944 & 0.008 & 5.8 \\
   7 & he        & 3 & 0.10 & 0.385 & 0.385 & 0.533 & 0.039 & 3.8 \\
   8 & he        & 3 & 0.30 & 0.911 & 0.911 & 0.611 & 0.042 & 7.2 \\
   9 & matchgate & 1 & 0.10 & 0.083 & 0.087 & 0.136 & 0.014 & 3.7 \\
  10 & matchgate & 1 & 0.30 & 0.396 & 0.400 & 0.427 & 0.056 & 0.5 \\
  11 & matchgate & 2 & 0.10 & 0.070 & 0.073 & 0.117 & 0.018 & 2.5 \\
  12 & matchgate & 2 & 0.30 & 0.314 & 0.312 & 0.328 & 0.066 & 0.2 \\
  13 & matchgate & 2 & 0.50 & 0.531 & 0.531 & 0.539 & 0.077 & 0.1 \\
  14 & matchgate & 2 & 0.80 & 0.758 & 0.757 & 0.755 & 0.054 & 0.1 \\
  15 & matchgate & 3 & 0.10 & 0.114 & 0.117 & 0.147 & 0.022 & 1.5 \\
  16 & matchgate & 3 & 0.30 & 0.490 & 0.490 & 0.498 & 0.066 & 0.1 \\
  17 & iqp       & 1 & 0.05 & 0.118 & 0.119 & 0.211 & 0.030 & 3.1 \\
  18 & iqp       & 1 & 0.10 & 0.320 & 0.319 & 0.386 & 0.075 & 0.9 \\
  19 & iqp       & 1 & 0.20 & 0.666 & 0.666 & 0.700 & 0.129 & 0.3 \\
  20 & iqp       & 1 & 0.30 & 0.838 & 0.837 & 0.843 & 0.125 & 0.0 \\
  21 & iqp       & 2 & 0.05 & 0.354 & 0.353 & 0.383 & 0.098 & 0.3 \\
  22 & iqp       & 2 & 0.10 & 0.710 & 0.709 & 0.720 & 0.167 & 0.1 \\
  23 & iqp       & 2 & 0.20 & 0.928 & 0.927 & 0.930 & 0.132 & 0.0 \\
  24 & iqp       & 2 & 0.30 & 0.977 & 0.977 & 0.978 & 0.083 & 0.0 \\
  \bottomrule
  \end{tabular}
\end{table}

\section{\texorpdfstring{Scaling to $n = 100$}{Scaling to n=100}}
\label{app:scaleup}

To address concerns about the modest training-set size of the main
experiments ($n = 30$ at $n_q = 6$ and $n = 40$ at $n_q = 8$), we re-ran
the M1 sweep (\textsc{he} family, $n_q = 6$, depths $L \in \{1,2,3,5,8,12\}$,
scales $s \in \{0.1, 0.5, 1, 2\}$) at $n = 100$. The full sweep took
${\approx}3.3$ hours on a single A6000 with \texttt{lightning.qubit}. The
qualitative trends are unchanged but two quantitative shifts are worth
noting (Figure~\ref{fig:scaleup}):
\begin{enumerate}
  \item The \textbf{NLL-optimal configuration $(L, s) = (1, 0.5)$ is
        invariant} between $n = 30$ and $n = 100$.
  \item The \textbf{best NLL improves} from $+0.89$ at $n = 30$ to
        $+0.46$ at $n = 100$, as expected from more training data.
  \item At fixed $(L, s)$ the normalized spectral entropy
        $S(K)/\log n$ \textbf{decreases} when $n$ grows (right panel),
        because $\log n$ grows faster than the spectral entropy itself.
        Consequently the absolute sweet-spot $S(K^*)/\log n$ shifts from
        $0.91$ at $n = 30$ to $0.79$ at $n = 100$.
\end{enumerate}
Combined with the $n_q$-dependent shift discussed in Section~5
(0.91 at $n_q = 6$ versus 0.99 at $n_q = 8$), this confirms that
Corollary~\ref{cor:target-opt}'s interval-valued prediction is the right
object to compare to: the absolute sweet-spot value depends on both the
training-set size $n$ and the qubit count $n_q$ via
$d_{\mathrm{int}}(y, K^*) / \log n$.

\begin{figure}[h]
  \centering
  \includegraphics[width=0.85\linewidth]{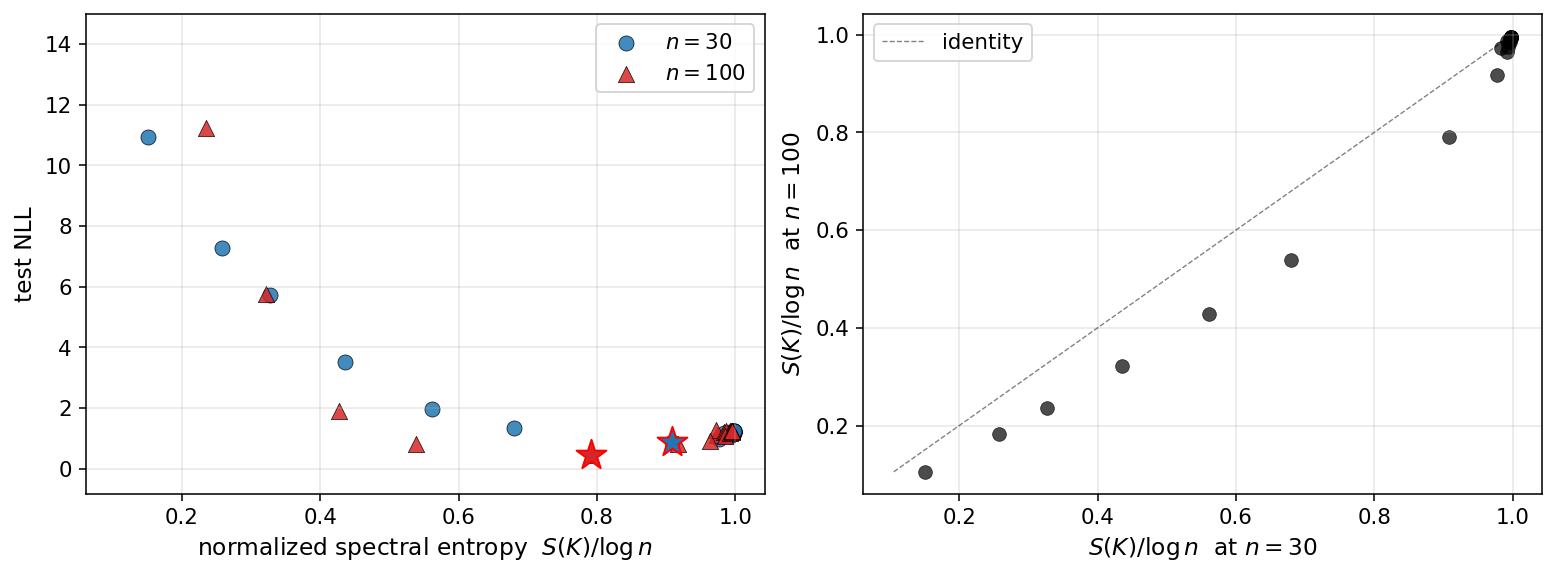}
  \caption{Scaleup from $n = 30$ to $n = 100$ on the synthetic
  target, $n_q = 6$, HE family. (Left) The same NLL U-shape; stars
  mark the best-NLL configuration in each sweep, which is identical
  in $(L, s)$ space but lies at a different $S(K)/\log n$ value
  (0.79 vs.\ 0.91) because $\log n$ grew. (Right) Per-configuration
  $S(K)/\log n$ at $n = 100$ vs.\ $n = 30$; all points lie below the
  identity line, confirming the $\log n$ normalization effect.}
  \label{fig:scaleup}
\end{figure}

\section{Spectral Regularization Experiments}
\label{app:regularization}

The two pathological regimes admit complementary spectral interventions.
For each we apply the regularizer consistently to the training Gram, the
train--test cross-covariance, and the test self-covariance, and ablate the
single tuning knob via held-out NLL.

\paragraph{Shrinkage rescues constant collapse.}
Define the modified kernel
$\Kq'(x, x') = (1 - \alpha)\, \Kq(x, x') + \alpha\, t\, \delta(x, x')$ where
$t = \mathrm{tr}(K)/n$ and $\delta$ is the Kronecker delta. This raises
$\Seff/\log n$ toward the useful middle while preserving positive
definiteness. Applied to the worst constant-collapse configuration in our
M1 sweep (\textsc{he}, $L = 1$, $s = 0.10$; baseline
$\Seff/\log n = 0.15$, NLL $= +10.94$), shrinkage at $\alpha = 0.30$ lifts
$\Seff/\log n$ to $0.54$ and drops NLL to $+0.94$ (Table~\ref{tab:reg}),
a $12\times$ improvement that places the regularized kernel near the
useful-hardness sweet spot of the unconstrained sweep (Figure~\ref{fig:reg},
left).

\paragraph{Truncation only mildly helps Haar concentration.}
Define $K_{\mathrm{trunc}}$ by zeroing all but the top-$k$ eigencomponents
of $K_{\mathrm{tr}}$ and projecting cross-covariances onto the same
top-$k$ subspace. Applied to a Haar-like configuration
(\textsc{he}, $L = 5$, $s = 1.5$; baseline $\Seff/\log n = 0.998$, NLL $=
+1.21$), the best truncation at $k = 15$ yields a marginal NLL improvement
to $+1.16$ (Table~\ref{tab:reg}). The lack of strong rescue is consistent
with our theoretical claim: information was never injected into the
posterior in the first place, so no post-hoc transform can recover it.
The lesson is operational---fix the kernel \emph{before} fitting if
$\Seff/\log n \ge 0.95$ (Algorithm~\ref{alg:recipe}), rather than rely on
spectral truncation.

\begin{table}[h]
\centering
\caption{Spectral regularization trajectories. Top: shrinkage on
constant-collapse (\textsc{he}, $L = 1$, $s = 0.10$). Bottom: truncation
on Haar-like (\textsc{he}, $L = 5$, $s = 1.5$). Best NLL row in bold.}
\label{tab:reg}
\small
\begin{tabular}{r r r r r}
\toprule
$\alpha$ & $\Seff/\log n$ & NLL & ECE & VC \\
\midrule
0.00 & 0.151 & $+10.94$ & 0.377 & 0.016 \\
0.05 & 0.236 & $+2.07$ & 0.213 & 0.075 \\
0.10 & 0.307 & $+1.25$ & 0.160 & 0.132 \\
0.20 & 0.432 & $+0.95$ & 0.123 & 0.241 \\
\textbf{0.30} & \textbf{0.542} & $\mathbf{+0.94}$ & \textbf{0.107} & \textbf{0.347} \\
0.50 & 0.729 & $+1.02$ & 0.083 & 0.551 \\
0.70 & 0.877 & $+1.12$ & 0.090 & 0.747 \\
0.90 & 0.979 & $+1.20$ & 0.090 & 0.933 \\
\midrule
$k$ & $\Seff/\log n$ & NLL & ECE & VC \\
\midrule
$30$ (full) & 0.998 & $+1.21$ & 0.097 & 0.987 \\
$25$ & 0.945 & $+1.19$ & 0.107 & 0.986 \\
$20$ & 0.880 & $+1.17$ & 0.113 & 0.985 \\
$\mathbf{15}$ & \textbf{0.795} & $\mathbf{+1.16}$ & \textbf{0.137} & \textbf{0.983} \\
$10$ & 0.676 & $+1.22$ & 0.210 & 0.979 \\
$5$  & 0.471 & $+1.68$ & 0.320 & 0.967 \\
$3$  & 0.321 & $+2.44$ & 0.373 & 0.953 \\
$1$  & 0.000 & $+6.45$ & 0.433 & 0.898 \\
\bottomrule
\end{tabular}
\end{table}

\begin{figure}[h]
  \centering
  \includegraphics[width=\linewidth]{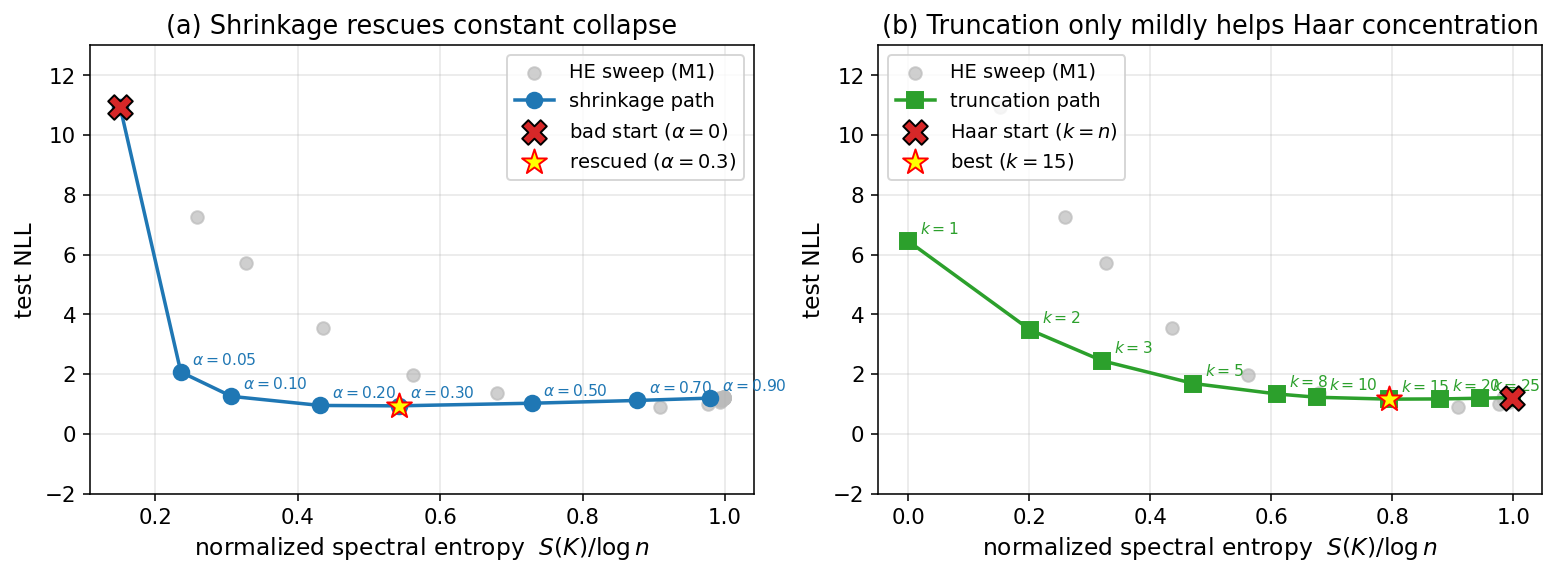}
  \caption{Spectral regularization trajectories on the M1 sweep manifold
  (gray points). (a) Shrinkage applied to the constant-collapse kernel
  (red X) sweeps it through the useful-hardness frontier, recovering the
  sweet-spot NLL at $\alpha \approx 0.30$ (yellow star). (b) Truncation
  applied to a Haar-like kernel (red X) only mildly improves NLL at
  $k = 15$, confirming that information loss at Haar concentration is
  largely irreversible.}
  \label{fig:reg}
\end{figure}

\end{document}